\documentclass[12pt, dvipsnames]{l4dc2020-mod} 

\title[Constrained Upper Confidence Reinforcement Learning]{Constrained Upper Confidence Reinforcement Learning}

\usepackage{times}
\usepackage{enumitem}

\usepackage{algorithm}
\usepackage{xcolor}

\SetCommentSty{mycommfont}

\author{%
 \Name{Liyuan Zheng, Lillian J. Ratliff} \Email{liyuanz8, ratliffl@uw.edu}\\
 \addr Department of Electrical \& Computer Engineering, University of Washington
}
\newcommand{\CUCRL}{{\tt C-UCRL}}
\newcommand{\UCRL}{{\tt UCRL2}}
\newcommand{\RSUCRL}{{\tt RS-UCRL2}}

\begin{document}

\maketitle

\begin{abstract}%
    Constrained Markov Decision Processes are a class of stochastic decision problems in which the decision maker must select a policy that satisfies auxiliary cost constraints. 
This paper extends upper confidence reinforcement learning for  settings in which the reward function and the  constraints, described by cost functions, are unknown a priori but the transition kernel is known. Such a setting is well-motivated by a number of applications including exploration of unknown, potentially unsafe, environments.  
We present an algorithm {\CUCRL} and show that it achieves sub-linear regret ($ O(T^{\frac{3}{4}}\sqrt{\log(T/\delta)})$) with respect to the reward while  satisfying the constraints even while learning with probability $1-\delta$. 
Illustrative examples are provided. 

\end{abstract}



\section{Introduction}
\label{sec:intro}
Markov Decision Processes (MDPs) have been successfully utilized to model sequential decision-making problems in stochastic environments. In the typical approach to learning a policy, the decision-maker trades off between exploration and exploitation, gradually improving their performance at the task as learning progresses. Reinforcement learning, a standard paradigm of learning in MDPs, has shown exceptional success in a variety of domains such as video games \citep{mnih2015human}, robotics \citep{lillicrap2015continuous, levine2016end}, recommender systems \citep{shani2005mdp}, autonomous vehicles \citep{sallab2017deep}, among many others. 

However, in many of these real-world applications, there is often additional constraints, or specifications that lead to constraints, on the learning problem. 
For instance, a recommender system should avoid presenting offending items to users and autonomous vehicles must avoid crashing into others while navigating \citep{garcia2015comprehensive}. 
Building algorithms that respect  safety constraints not only during normal operation, but also during the initial learning period, is a question of particular interest \citep{leike2017ai}. This problem is known as the \emph{safe exploration problem} \citep{moldovan2012safe, amodei2016concrete}.
 In the standard MDP framework, an approach for baseline performance is risk-sensitive reinforcement learning \citep{coraluppi1999risk, garcia2015comprehensive}, where the optimization criterion is transformed in order to reflect a subjective measure balancing the return and the risk.

On the other hand, in a safety-critical environment, it is more reasonable to separate the return and the risk criterion, and enforce constraint satisfaction in the learning procedure. A standard formulation for an environment with safety constraints is the constrained MDPs (CMDPs) \citep{altman1999constrained}. A decision-maker facing a  CMDP aims to maximize the total reward while satisfying the constraints on costs in expectation over the whole trajectory. 

In recent literature, policy gradient-based reinforcement learning algorithms have been proposed as a means to learn a policy for a CMDP. The following are two constrained policy search algorithms with state-of-the-art performance guarantees: Lagrangian-based actor-critic algorithm \citep{bhatnagar2012online, chow2018risk} and Constrained Policy Optimization (CPO) \citep{achiam2017constrained}. However, for these policy gradient-based methods, safety is only approximately guaranteed \emph{after} a sufficient learning period. The fundamental issue is that without a model, safety must be learned via trial and error, which means it may be violated during initial learning interactions.

Model-based approaches have utilized Gaussian processes to model the state safety values or the dynamic uncertainties \citep{berkenkamp2017safe, koller2018learning, wachi2018safe, cheng2019end} or utilized Lyapunov-based methods \citep{chow2018lyapunov} to guarantee safety during learning. Although these methods guarantee constraint satisfaction during learning, an arguably valuable analysis of the regret  is lacking.

In unconstrained settings when the reward and transition kernel are unknown, upper confidence based reinforcement learning algorithms have been proposed---namely, {\UCRL} \citep{jaksch2010near}---with sub-linear regret. The key idea is to build confidence intervals on the reward and transition kernel and iteratively solve for policies using value iteration based methods.

In this work, we are not only interested in learning the optimal policy that satisfies the constraints via interacting with the stochastic environment, but also in ensuring performance guarantees on the learning algorithm during learning. With some practical scenarios in mind, we make the assumption that the rewards and constraint costs are unknown.
  For instance, consider a rover exploring the Mars landscape; here one can model the dynamics of the rover as known with some uncertainty and the reward and constraints which model the value of exploring the environment as unknown---e.g., constraints can be abstracted as costs which seek to limit the frequency of visiting a potentially hazardous states \citep{el2019controlled}.
 
 Motivated by upper confidence reinforcement learning 
 \citep{jaksch2010near}, we introduce the constrained upper confidence reinforcement learning (\CUCRL) algorithm which combines elements of the classical {\UCRL} algorithm with robust linear programming\footnote{We remark that {\UCRL} assumes the transition kernel is unknown a priori where we assume it is known; we leave extending our approach to unknown dynamics to future work.}.
We define our goals as follows: (1) maintain constraint satisfaction throughout the learning process with high probability, and (2) achieve sub-linear regret comparing the rewards collected by the algorithm during learning with the reward of an optimal stochastic policy.

\paragraph{Contributions.} The contributions can be summarized as follows.
 Building on {\UCRL}, we introduce the {\CUCRL} algorithm (Algorithm~\ref{alg:safeUCRL}).
 We show that {\CUCRL} is guaranteed to satisfy constraints during learning with probability at least $1-\delta$ (Theorem~\ref{thm:deltasafe}) and achieves $ O(T^{\frac{3}{4}}\sqrt{\log(T/\delta)})$ reward regret (Theorem~\ref{thm:regret}).
   Of independent interest, we note that when the state space is trivial, the setting we consider subsumes stochastic multi-armed bandits with per-round budget constraints, where the optimal policy is a randomized policy across arms.

\paragraph{Organization.}
The rest of the paper is organized as follows. An overview of related work  is
provided in Section~\ref{sec:related_work}.
 Mathematical preliminaries and our algorithm are introduced in Section \ref{sec:algo}. Analysis of both constraints satisfaction and reward regret is provided in Section \ref{sec:analysis}.
Several illustrative examples are provided in Section \ref{sec:experiment}.
In those experiments we compare our proposed method to Risk-Sensitive {\UCRL} algorithm and show that {\UCRL} algorithm fails to converge to the optimal policy while our algorithm does.
The paper is concluded in Section \ref{sec:conclusion} with a brief summary and
discussion of future directions. 


\section{Related Work}
\label{sec:related_work}

Recently, several policy gradient-based reinforcement learning algorithms have been proposed
for learning policies for CMDPs.  In particular, there are two noteable
constrained policy search algorithms which enjoy state-of-the-art performance:
a Lagrangian-based algorithm \citep{bhatnagar2012online, chow2018risk} and
Constrained Policy Optimization (CPO) \citep{achiam2017constrained}. The
Lagrangian-based algorithm formulates the CMDP problem as a minimax problem and
uses primal-dual gradient optimization to find the saddle point solution. 
While this procedure will asymptotically converge to the saddle point solution,
in general there is no guarantee on policies being safe during the learning procedure. 
On the other hand, CPO---a method that derives from an extension of trust-region
policy optimization (TRPO)---guarantees monotonic performance improvements on
the expected reward and a guarantee on constraint satisfaction throughout training.
While this algorithm is safe during learning, analyzing its convergence is challenging and the regret analysis with respect to reward is lacking.


As an alternative to  policy gradient reinforcement learning algorithms, linear
programming based algorithms have been proposed. In \citep{el2019controlled}, CMDPs with known reward, constraints, transition kernel but uncertain initial state distribution are considered. Linear programming based algorithms are proposed to solve for safe policies in this setting. In our setting, however, the reward and constraints are stochastic and considered  unknown a priori, which the stochastic transition kernel is known.


Most similar to our approach is {\UCRL}; in particular, our approach can be
viewed as an extension of {\UCRL} \citep{jaksch2010near}, in some sense, by
incorporating constraints; the one difference is that we assume the transition
kernel is known while the classical {\UCRL} algorithm does not. We leave
extending our setting to unknown transition kernels to future work. As alluded
to in the introduction, in {\UCRL}, the reward and transition kernel are
approximated and the policy is obtained by value iteration based methods in a “optimism in the face of
uncertainty” fashion.
Further, the performance of {\UCRL} is
analyzed by bounding the regret with respect to the optimal \emph{deterministic}
policy. CMDPs, however, in general do not admit deterministic policies. In
{\CUCRL}, the reward and constraints are approximated and the policy is
obtained by solving a robust linear program. Performance is assessed by
computing the \emph{reward} regret with respect to the optimal randomized policy.

Finally, our work is related to the multi-armed bandit problem with  constraints. 
Previous works, e.g., have considered the multi-armed bandit problem with an auxiliary cost in addition to the traditonal reward
\citep{ding2013multi, zhou2018budget}. The `game' (between the player and the environment) ends when the sum of current costs associated with the played arms exceeds the remaining budget, which is fixed and known to the player. The typical approach is to  construct upper confidence bounds for the reward-to-cost ratio and then utilize them in upper confidence bound-based algorithms. On the other hand,  in our approach, we use upper confidence bounds for both reward and cost, and solve a linear program to obtain the policy policy. In related work,  fairness constraints are incorporated into a multi-armed bandit setting; in particular, arms that are perceived to have less value/reward should never favored over better performing alternatives, despite a learning algorithm's uncertainty over the true payoffs \citep{joseph2016fairness}. In such settings, the algorithm is forced to pick arms uniformly until the player has enough confidence of the performance of arms. 
Connecting to this body of work, our problem reduces to a constrained multi-armed bandit problem when there is a single state. 
 The main difference between our setting and that of the majority existing multi-armed bandit literature with constraints is that the optimal policy and policies obtainable by our algorithm can be a randomized or stochastic policy as opposed to a deterministic `best arm' policy. 

\section{Constrained Upper Confidence Reinforcement Learning Algorithm}
\label{sec:algo}


An MDP is a tuple $(\mathcal{S}, \mathcal{A}, P, r)$, where $\mathcal{S}$ is the set of states, $\mathcal{A}$ is the set of actions, $P: \mathcal{S} \times \mathcal{A} \times \mathcal{A} \rightarrow [0,1]$ is the transition kernel such that $P(s'|s,a)$ is the probability of transitioning to state $s'$ given that the previous state was $s$ and the agent took action $a$ in $s$, and $r: \mathcal{S} \times \mathcal{A} \rightarrow [0,1]$ is the reward function. A stationary policy $\pi: \mathcal{S} \times \mathcal{A} \rightarrow [0,1]$ is a map from states to a probability distribution over actions, with $\pi(a|s)$ denoting the probability of selecting action $a$ in state $s$. We consider the setting in which the transition kernel $P(s'|s,a)$ is known to the agent, but the reward and costs are stochastic and unknown. In the example of a rover exploring the surface of Mars, the agent (rover) is aware of the transition probability of next state based on its action, but the \emph{safety quality} of each state is unknown.
Let $S = |\mathcal{S}|$ and $A = |\mathcal{A}|$ where $|\cdot|$ is the cardinality of its argument. We use the notation $[\cdot]=\{1,\ldots, \cdot\}$ for index sets.

\subsection{Constrained Markov Decision Processes}
A CMDP is an MDP augmented with `cost' constraints that restrict the set of allowable policies for that MDP. 
For a given CMDP, we consider the performance measure to be the \emph{infinite horizon average reward} which is given by
\begin{equation}
J(\pi) =\textstyle \lim_{T \to \infty} \mathbb{E}_{\tau \sim \pi} \big[ \frac{1}{T} \sum_{t=0}^{T-1} r(s_t, a_t) \big]
\end{equation}
where $\tau$ denotes a trajectory $\tau = (s_0, a_0, s_1, \dots)$, and $\tau \sim \pi$ is shorthand for indicating that the distribution over trajectories depends on $\pi$: $s_0 \sim p(s_0), a_t \sim \pi(\cdot|s_t), s_{t+1} \sim P(\cdot|s_t, a_t)$. Similarly, define the \emph{average constraint costs} by
\begin{equation}
C_i(\pi) = \textstyle\lim_{T \to \infty} \mathbb{E}_{\tau \sim \pi} \big[ \frac{1}{T} \sum_{t=0}^{T-1} c_i(s_t, a_t) \big].
\end{equation}
where $\{c_1, \dots, c_m\}$ with $c_i: \mathcal{S} \times \mathcal{A} \rightarrow [0,1]$ are the cost constraints. 
The CMDP is then defined by
\begin{equation}
\textstyle\max_{\pi} \ \{J(\pi)|\  C_i(\pi) \le d_i, \ \ \forall\ i \in [m]\}
\end{equation}
where $\{d_1, \ldots, d_m\}$ are upper bounds on the average constraint costs.  Note that without loss of generality both the reward and costs are random variables with a distribution supported on $[0,1]$.

Denote the mean of reward and cost constraint functions as $\bar{r}(s, a) =
\mathbb{E}[r(s,a)]$, $\bar{c}_i(s,a) = \mathbb{E}[c_i(s,a)]$ where the expectation is taken with respect to the distribution of the reward and cost function of that state-action pair $(s,a)$. 
If the transition kernel $P(s'|s,a)$, the mean of the reward function $\bar{r}(s, a)$, and mean cost functions $\bar{c}_i(s,a)$ are all given, them we can solve the CMDP by solving the following linear program \citep{altman1999constrained}: 
\begin{align*}
\max_{y} & \quad \textstyle\sum_{s, a} \bar{r}(s, a) y(s, a) \\
\text{s.t.} & \textstyle\quad \sum_{a'} y(s', a') = \sum_{s, a} P(s'|s, a) y(s, a) \\
& \quad \textstyle\sum_{s, a} y(s, a) = 1, \quad y(s, a) \ge 0 \\
& \quad \textstyle\sum_{s, a} \bar{c}_i(s, a) y(s, a) \le d_i, \ i\in [m]
\end{align*}
To simplify notation, we write the above linear program in matrix form as
follows:
\begin{equation}
\textstyle\max_{y} \{ \  \bar{r}^\top y \ |\  I_o y = P y, \  \mathbf{1}^\top y = 1,\  y \ge 0, \ \bar{c}^\top y \le d\}
\end{equation}
where $\bar{r} \in \mathbb{R}^{SA}, y \in \mathbb{R}^{SA}$, $\bar{c} \in \mathbb{R}^{SA\times m}, d \in \mathbb{R}^m$, $P \in \mathbb{R}^{S \times SA}$, and  $I_o \in \mathbb{R}^{S \times SA}$ is a sparse matrix built by placing $S$ row blocks of length $A$ in a block diagonal fashion, where each row block consists of all ones. 
Here, $y \in \mathbb{R}^{S \times A}$ represents the steady-state occupation measure defined by
\begin{equation}
y(s, a) = \textstyle\lim_{T \to \infty} \mathbb{E}_{\tau \sim \pi} \big[\frac{1}{T} \sum_{t=0}^{T-1} \mathbf{1} \{ s_t=s, a_t=a \} \big].
\end{equation}
With $\bar{y}$ the solution of this linear program, the optimal  stationary policy is
\begin{equation}
\bar{\pi}(a|s) = \textstyle\bar{y}(s, a)/(\sum_{a \in \mathcal{A}} \bar{y}(s, a)).
\label{eq:policyrecover}
\end{equation}

\noindent\textbf{Remark}. \emph{It is worth noting that unlike in tabular MDPs without
    constraints, where the optimal policy is always deterministic, the optimal
    policy in CMDPs could be stochastic \citep{puterman2014markov}.  It is, in
    fact, trivial to solve the CMDP if the optimal policy in CMDPs is
    deterministic because that means the constraints are not active.}

\begin{algorithm}[t]
	\SetAlgoLined
	\KwIn{safety parameter $\delta \in (0,1)$, baseline policy $\pi_0(a|s)$, episode length $h$.}
	{\textbf{Initialization:} set $t=1$, observe the initial state $s_1$ }
	
	\For{$\mathrm{episodes}$ $k=1, 2, \dots, K$}{
	    $t_k = t$ \tcp*{initialize start time of episode $k$ }
	    \While{$t\leq t_k + h$\tcp*{Execute baseline policy $h$ times for exploration}}{  

			Draw action $a_t \sim \pi_0(\cdot|s_t)$ 
			
			Observe reward $r_t$, costs $c_{i,t}$, and the next state $s_{t+1}$ 
			
			$t\gets t+1$
			
			
		}

		$N_k(s,a) = \textstyle\sum_{t' = 1}^{t} \mathbf{1}(s_{t'} = a, a_{t'} = a), \ \forall (s,a)\in \mathcal{S}\times\mathcal{A}$ \tcp*{set the state-action count}
		
 $R_k(s, a) =  \textstyle\sum_{t' = 1}^{t} r_{t'} \mathbf{1}(s_{t'} = a, a_{t'} = a) $\tcp*{compute cumulative reward}
		
$C_{i, k}(s, a) =  \textstyle\sum_{t' = 1}^{t} c_{i,{t'}} \mathbf{1}(s_{t'} = a, a_{t'} = a) $\tcp*{compute the cumulative costs}
		
 $\hat{r}_k(s, a) = \textstyle\frac{R_k(s, a)}{\max \{1, N_k(s, a)\}}, \quad \hat{c}_{i,k}(s, a) = \frac{C_{i, k}(s, a)}{\max \{1, N_k(s, a)\}}$\tcp*{compute estimates}

$\tilde{y}_k\gets \arg\max$ of 	\eqref{eq:RLP} using $\tilde{r}_k(s,a)$ and $\tilde{c}_{i,k}(s,a)$ in \eqref{eq:tilder} and \eqref{eq:tildec}, resp.

$\tilde{\pi}_k\gets \tilde{y}_k(s,a)/(\sum_{a\in \mathcal{A}}\tilde{y}_k(s,a))$ \tcp*{recover policy}


	   \While{$t\leq t_k + kh$ \tcp*{Execute $\tilde{\pi}_k$ policy $(k-1)h$ times}}{
	   	   
	   	    Draw action $a_t \sim \tilde{\pi}_k(\cdot|s_t)$ 
	   	    
	   	    Observe reward $r_t$, costs $c_{i,t}$, and the next state $s_{t+1}$
	   	    
	   	    $t\gets t+1$
	   	    
	   }
	}
\caption{Constrained UCRL ({\CUCRL}) algorithm}
\label{alg:safeUCRL}
\end{algorithm}

\subsection{Constrained Upper Confidence Reinforcement Learning Algorithm}
Since the reward and constraint cost functions are unknown, motivated by {\UCRL}, we introduce {\CUCRL} (Algorithm~\ref{alg:safeUCRL}). In general, the {\CUCRL} algorithm follows a principle of ``optimism in the face of reward uncertainty; pessimism in the face of cost uncertainty." That is, it defines confidence intervals for the reward and cost of each state-action pair given the observations so far, and solves for the optimistic policy that satisfies the constraints.
More specifically, in {\CUCRL}, given the current confidence interval estimates, we use a robust linear program \citep{luenberger1984linear} formulation to find a policy using the confidence intervals as determined at the current iteration. 

In particular, in episode $k$, we start by executing the baseline policy $\pi_0$ for a constant $h$ number of iterations\footnote{The heuristic for choosing $h$ is based on the mixing time of the Markov chain induced by $\pi_0$ given the known transition kernel for the CMDP.}.  It is common to assume a initial safe baseline policy  \citep{achiam2017constrained} and without loss of generality, we assume under such policy, the Markov chain resulting from the CMDP is irreducible and aperiodic \citep{bhatnagar2009natural}. This baseline policy could, e.g., be obtained by some prior information about which states are safe to start the conservative exploration\footnote{Choosing $\pi_0$ is an important component of {\CUCRL}. In Section~\ref{sec:experiment}, we provide some intuitive choices for the simple examples we present, while we leave further development on how to select $\pi_0$, either heuristically or theoretically, to future work.}. After executing $\pi_0$,  we define 
estimates of the reward and costs by
\[\textstyle\hat{r}_k(s, a) = \textstyle\frac{R_k(s, a)}{\max \{1, N_k(s, a)\}}\] and \[\textstyle\hat{c}_{i,k}(s, a) = \frac{C_{i, k}(s, a)}{\max \{1, N_k(s, a)\}},\]
 respectively, where $N_k(s,a)$, $R_k(s,a)$, and $C_{i,k}(s,a)$ are the  state-action count, and cumulative reward and costs, respectively, as defined in Algorithm~\ref{alg:safeUCRL}. The visitation frequency random variable $N_k(s,a)$ is defined to be the sum of indicators of whether or not the state-action pair $(s,a)$ was visited in each iteration over all episodes. 
 The corresponding reward $R_k(s,a)$ and constraint costs $C_{i,k}(s,a)$ are defined similarly.

 Using these estimates, we define 
\begin{equation}
\tilde{r}_k(s, a)  = \textstyle\min \big\{ \hat{r}_k(s,a) + \big(\frac{\log (S A (m+1) \pi^2 t_k^3 / 3 \delta)}{2 \max \{1, N_k(s,a)\}}\big)^{1/2}, 1 \big\} 
\label{eq:tilder}\end{equation}
and
\begin{equation}
\tilde{c}_{i,k}(s, a)  =\textstyle \min \big\{ \hat{c}_{i,k}(s,a) + \big(\frac{\log (S A (m+1) \pi^2 t_k^3 / 3 \delta)}{2 \max \{1, N_k(s,a)\}}\big)^{1/2}, 1 \big\},
\label{eq:tildec}
\end{equation}
where 
\[\textstyle\big(\frac{\log (S A (m+1) \pi^2 t_k^3 / 3 \delta)}{2 \max \{1, N_k(s,a)\}}\big)^{1/2}\]defines the confidence interval as we show in Section~\ref{sec:analysis}.
We then use \eqref{eq:tilder} and \eqref{eq:tildec} to
define  the following robust linear program:
\begin{equation}
\textstyle \max_{y} \{ \ \tilde{r}_k^\top y \ | \ I_o y = P y, \ \mathbf{1}^\top
y = 1, \ y \ge 0, \ \tilde{c}_k^\top y \le d\}. \tag{RLP}\label{eq:RLP}
\end{equation}
A few comments here on guaranteeing that the feasible set is non-trivial are warranted. 
Our analysis results are predicated on $\pi_0$ and $h$ being chosen such that in each episode the robust linear program we solve has at least one feasible solution. The duration $h$ is chosen based on the mixing time of the induced Markov chain under the baseline policy with the goal of ensuring with high probability that the feasible set is not empty; for instance, `sufficient' exploration will guarantee that $\tilde{c}_1^\top y\leq d$ for some $y\in \{I_oy=Py, \ \boldsymbol{1}^\top y=1, y\geq 0\}$. It is possible that in the first episode, even after $h$ iterations of executing the baseline policy, that there is no $y$ such that $\tilde{c}_{1}^\top y\leq d$\footnote{ e.g., if $\tilde{c}_{i,1}(s,a)=1$ for each state-action pair and constraint $i\in [m]$, then clearly the feasible set is empty if $d>1$.}. A heuristic we use in practice is to run the baseline policy $\pi_0$ for as many iterations as it takes for $y_0 \in \{I_oy=Py, \ \boldsymbol{1}^\top y=1,\ y\geq 0,\ \tilde{c}_1^\top y\leq d\}$. Then, we are guaranteed that in all future episodes,  $y_0$ is always in the feasible set of \eqref{eq:RLP}.  We leave further exploration of theoretically guaranteeing that the \eqref{eq:RLP} has a non-trivial feasible set in the first episode to  future work.

Returning to the description of the algorithm, in episode $k$, the solution $\tilde{y}_k$ to the robust linear program is then used to construct the policy $\tilde{\pi}_k$ via \eqref{eq:policyrecover}. This policy is executed for a linearly increasing number of iterations $(k-1)h$ where $k$ is the episode index and $h$ is the fixed duration used for executing the baseline policy.  To summarize, for each episode of {\CUCRL}, we execute the baseline policy for $h$ steps, estimate the reward and costs, and then execute $\tilde{\pi}_k$ for a linearly increasing (in the number of epochs) number of steps $(k-1)h$, making $kh$ the total duration of episode $k$.



\section{Analysis: Regret Bounds and High-Probability Safety Guarantees}
\label{sec:analysis}

In this section, we summarize our analysis results. We first show that {\CUCRL} has guarantees on constraint satisfaction during learning. Then, we provide regret analysis with respect to the reward, showing that the regret is sub-linear. 

\subsection{Constraint/Safety Guarantees}
To capture constraint satisfaction, we leverage the notion of $\delta$-safety. 

\begin{definition}[$\delta$-safe]
	An algorithm is $\delta$-safe if, with probability at least $1-\delta$, for all time steps $t$, the policy executed by the algorithm satisfies
$C_i(\pi_t) \le d_i$, $\forall i \in [m]$.
\end{definition}

Following~\cite{jaksch2010near}, we define the set of \emph{plausible CMDPs} by
the confidence intervals for the reward and each of the constraint costs. In
particular, at
episode $k$, let $\mathcal{M}_k$ be the set of plausible CMDPs with states and
actions as in the underlying true CMDP $M$, define by all such CMDPs satisfying
the following:
	\begin{align}
	| \hat{r}_k(s,a) - \bar{r}(s,a)| & \le \textstyle\big(\frac{\log (S A (m+1)
    \pi^2 t_k^3 / 3 \delta)}{2 \max \{1, N_k(s,a)\}}\big)^{1/2},\label{eq:confr} \\
	|\hat{c}_{i,k}(s,a) - \bar{c}_i(s,a)| & \textstyle\le \big(\frac{\log (S A
    (m+1) \pi^2 t_k^3 / 3 \delta)}{2 \max \{1,
    N_k(s,a)\}}\big)^{1/2}, \ i\in [m]\label{eq:confc}
	\end{align}
    for all state-action pairs $(s,a)\in \mathcal{S}\times \mathcal{A}$. Let
    $\mathcal{M}$ be the set of plausible for all episodes $k$.
\begin{lemma} \label{lemma:unionbound}
  For any fixed $k\geq 1$, the probability that the true CMDP $M$ is not contained in
  the set of plausible CMDPs $\mathcal{M}_k$ at episode $k$ is at most $6\delta/(\pi^2 t_k^2)$.
Furthermore, with probability at least $1-\delta$, for every state-action pair $(s,a)$, cost $c_i$ and episode $k$,
	{\CUCRL} satisfies the following:
	\begin{align}
	| \hat{r}_k(s,a) - \bar{r}(s,a)| & \le \textstyle\big(\frac{\log (S A (m+1)
    \pi^2 t_k^3 / 3 \delta)}{2 \max \{1, N_k(s,a)\}}\big)^{1/2},\label{eq:confr} \\
	|\hat{c}_{i,k}(s,a) - \bar{c}_i(s,a)| & \textstyle\le \big(\frac{\log (S A
    (m+1) \pi^2 t_k^3 / 3 \delta)}{2 \max \{1,
    N_k(s,a)\}}\big)^{1/2}\label{eq:confc}
	\end{align}
    Hence, the probability that the true CMDP $M$ is not in the set of all
    plausible CMDPs for any episode $k$ is at most $\delta$---that is,
    $\Pr\{M\notin \mathcal{M}\}\leq \delta$.
\end{lemma}
\begin{proof}
    Consider any fixed state-action pair $(s,a)$ and its visitation frequency 
 $N_k(s,a)$ up to episode $k$. If the state-action pair $(s,a)$ has not been
 visited, then \eqref{eq:confr} and
 \eqref{eq:confc} trivially hold since $N_k(s,a)=0$ by definition and the
 right-hand sides of  \eqref{eq:confr} and
 \eqref{eq:confc} are greater than one when $N_k(s,a)=0$. 

     On the other hand, if $N_k(s,a)$ is not zero, meaning the state-action pair
     has been visited, then since for each $(s,a)$ pair, the reward and constraint costs are all
    supported on $[0,1]$ and independent identically distributed (iid)
    real-valued random variables, we can apply Hoeffding's inequality to get a
    bound on the deviation between the true mean $\bar{r}(s,a)$ (respectively,
    $\bar{c}_i(s,a)$) and the empirical mean $\hat{r}_k(s,a)$ (respectively,
    $\hat{c}_{i,k}(s,a)$) given $n$ iid samples of the state-action pair $(s,a)$:
    \begin{equation}	\Pr \left\{ |\hat{r}_{k}(s,a) - \bar{r}(s,a)| \ge \epsilon \right\} \le 2 \exp (-2 n \epsilon^2)
        \label{eq:hoeffding}
    \end{equation}
Consider  
\[\textstyle \epsilon=\left(\frac{1}{2n}\log\left(\frac{SA(m+1)\pi^2t_k^3}{3\delta}
\right)\right)^{1/2},\]
then
\begin{align*}
    \textstyle\Pr\left\{|\hat{r}_{k}(s,a) - \bar{r}(s,a)| \ge \left(\frac{1}{2n}\log\left(\frac{SA(m+1)\pi^2t_k^3}{3\delta}
\right)\right)^{1/2}  \right\}  &\textstyle\leq 2\exp\left( -2n \frac{1}{2n}\log\left(\frac{SA(m+1)\pi^2t_k^3}{3\delta}
\right)\right)\\
&=\textstyle \frac{6 \delta}{S A (m+1) \pi^2 t_k^3}
\end{align*}
	Similarly, for each state-action pair $(s,a)$ and constraint cost indexed by $i$,
	\begin{equation}
\textstyle	\Pr \left\{ |\hat{c}_{i,k}(s,a) - \bar{c}_i(s,a)| \ge \left(\frac{1}{2n}\log\left(\frac{SA(m+1)\pi^2t_k^3}{3\delta}
\right)\right)^{1/2}\right\}
\le \frac{6 \delta}{S A (m+1) \pi^2 t_k^3}.
	\end{equation}    
Noting that from the above argument, the confidence intervals hold with
probability one when $(s,a)$ has not be visited, taking a union bound over all
possible values of $n \in\{ 1,\dots, t_k\}$ gives
\begin{equation*}
    	\textstyle \Pr \big\{  \cup_{n = 1}^{t_k}
            \big\{ |\hat{r}_{k}(s,a) - \bar{r}(s,a)| \ge \textstyle\big(\frac{\log (S A (m+1)
    \pi^2 t_k^3 / 3 \delta)}{2 \max \{1, N_k(s,a)\}}\big)^{1/2}
	    	\big\} \big\} \textstyle\le  \sum_{n = 1}^{t_k}  \frac{6
            \delta}{S A (m+1) \pi^2 t_k^3} =
            \frac{6\delta}{SA(m+1)\pi^2t_k^2}
        \end{equation*}
and 
\begin{equation*} \textstyle\Pr \big\{  \cup_{n = 1}^{t_k}
            \big\{ |\hat{c}_{i,k}(s,a) - \bar{c}_i(s,a)| \ge
         \textstyle\big(\frac{\log (S A (m+1)
    \pi^2 t_k^3 / 3 \delta)}{2 \max \{1, N_k(s,a)\}}\big)^{1/2} 
	    	 \big\} \big\} 
	    	\le  \sum_{n = 1}^{t_k}  \frac{6
            \delta}{S A (m+1) \pi^2 t_k^3} =
            \frac{m6\delta}{SA(m+1)\pi^2t_k^2}
        \end{equation*}
where we have now written $N_k(s,a)$ for the number of visits in $(s,a)$ up to
episode $k$.
This proves \eqref{eq:confr} and \eqref{eq:confc}.

   Now, further union bounding over  all state-action pairs $(s,a)$ 
gives 
\begin{equation}
    	\textstyle \Pr \big\{  \cup_{n = 1}^{t_k}
            \cup_{s,a} \big\{ |\hat{r}_{k}(s,a) - \bar{r}(s,a)| \ge
            \epsilon_r(n) 
	    	\big\} \big\} \textstyle\le  \sum_{n = 1}^{t_k} \sum_{s,a} \frac{6
            \delta}{S A (m+1) \pi^2 t_k^3} =
            \frac{6\delta}{(m+1)\pi^2t_k^2}\label{eq:rlittlebd}
        \end{equation}
        for the reward. 
Analogously, taking a further union bound over all state-action pairs $(s,a)$ and all constraint costs $i\in[m]$, 
gives \begin{equation} \textstyle\Pr \big\{  \cup_{n = 1}^{t_k}
             \cup_{s,a,i} \big\{ |\hat{c}_{i,k}(s,a) - \bar{c}_i(s,a)| \ge
             \epsilon_r(n) 
	    	 \big\} \big\} 
	    	\le  \sum_{n = 1}^{t_k} \sum_{s,a,i} \frac{6
            \delta}{S A (m+1) \pi^2 t_k^3} =
            \frac{m6\delta}{(m+1)\pi^2t_k^2}\label{eq:clittlebd}
        \end{equation}
        for the constraint costs.
Summing \eqref{eq:rlittlebd} and \eqref{eq:clittlebd}, we get the first claim of
the lemma---i.e., 
\[\textstyle\Pr\{M\not\in\mathcal{M}_k\}\leq \frac{6\delta}{\pi^2 t_k^2}\]
Now, since  $\sum_{\ell=1}^\infty \frac{1}{\ell^2} = \frac{\pi^2}{6}$, if in
\eqref{eq:rlittlebd} and \eqref{eq:clittlebd}, we additionally union bounded
over all episodes $k\in \{1, \ldots, \infty\}$, we get that
	\begin{equation*}
	    	\textstyle \Pr \big\{ \cup_{t_k = 1}^{\infty} \cup_{n = 1}^{t_k}
            \cup_{s,a} \big\{ |\hat{r}_{k}(s,a) - \bar{r}(s,a)| \ge
            \epsilon_r(n) 
	    	\big\} \big\} \textstyle\le \sum_{t_k=1}^{\infty} \sum_{n = 1}^{t_k} \sum_{s,a} \frac{6 \delta}{S A (m+1) \pi^2 t_k^3} = \frac{\delta}{m+1}
	\end{equation*}
	and
	\begin{equation*}
	    	 \textstyle\Pr \big\{ \cup_{t_k = 1}^{\infty} \cup_{n = 1}^{t_k}
             \cup_{s,a,i} \big\{ |\hat{c}_{i,k}(s,a) - \bar{c}_i(s,a)| \ge
             \epsilon_r(n) 
	    	 \big\} \big\} 
	    	\le \sum_{t_k=1}^{\infty} \sum_{n = 1}^{t_k} \sum_{s,a,i} \frac{6 \delta}{S A (m+1) \pi^2 t_k^3} = \frac{m\delta}{m+1}
	\end{equation*}
so that
\[\textstyle \Pr\{M\not\in \mathcal{M}\}\leq \delta\]
which proves the final statement in the lemma.
\end{proof}

Given that, for each episode, we can bound the gaps between the estimated reward (respectively, costs) and the mean reward (respectively, mean costs), with probability $1-\delta$, we can provide an assurance on {\CUCRL} being $\delta$-safe.
\begin{theorem}
{\CUCRL} is $\delta$-safe. 
\label{thm:deltasafe}
\end{theorem}

\begin{proof}
	According to Lemma~\ref{lemma:unionbound}, with probability at least $1-\delta$, $\bar{c}_i(s, a) \le \tilde{c}_{i,k}(s,a)$. The occupation measure $\tilde{y}_k$ obtained at each episode via \eqref{eq:RLP}
	satisfies $\sum_{s,a} \tilde{c}_{i,k}(s,a) \tilde{y}_k(s,a) \le d_i$.
Hence,	 $C_i(\tilde{\pi}_k) = \sum_{s,a} \bar{c}_{i}(s,a) \tilde{y}_k(s,a) \le d_i$ with probability $1-\delta$.
\end{proof}

\subsection{Regret Analysis of {\CUCRL}}
Given that we have shown that {\CUCRL} is $\delta$-safe,  we now analyze the reward regret. In episode $k$ of {\CUCRL}, we execute a baseline policy $\pi_0$ for $h$ times and policy $\tilde{\pi}_k$ for $(k-1)h$ times. The pseudo-regret of episode $k$ is given by
\begin{equation*}
\Delta_k  = h [J(\bar{\pi}) - J(\pi_0)] + (k-1)h [J(\bar{\pi}) - J(\tilde{\pi}_k)] = h \bar{r}^\top (\bar{y} - y_0) + (k-1)h \bar{r}^\top (\bar{y} - \tilde{y}_k).
\end{equation*}

We first upper bound the per-step pseudo-regret of executing policy $\tilde{\pi}_k$, $\bar{r}^\top (\bar{y} - \tilde{y}_k)$, where the first term is the expected average reward under the optimal policy $\bar{\pi}$ and the second term is the sub-optimal expected average reward under policy $\tilde{\pi}_k$.

Using the confidence bounds in Lemma~\ref{lemma:unionbound}, define
	\begin{equation}
	\epsilon_r(s,a)= \textstyle\big(\frac{\log (S A (m+1)
    \pi^2 t_k^3 / 3 \delta)}{2 \max \{1,
    N_k(s,a)\}}\big)^{1/2},\label{eq:epsilonr}
\end{equation}
and $\epsilon_c(s,a)=\epsilon_r(s,a)$ for each state-action pair and let
$\epsilon_r$ and $\epsilon_c$ denote the vectors containing the values across
all state-action pairs\footnote{We note that it is
    possibel to define separate confidence bounds for the reward and constraint
    costs, however, for simplicity of the statement and proof of
    Lemma~\ref{lemma:unionbound}, we define them to be the same.}.
Define the following two linear programs:
\begin{align}
    &\textstyle\max_y \{ r^\top y | Ay=0, \mathbf{1}^\top y = 1, y \ge 0, c^\top y \le d \}    \label{eq:y1}\\
    &\textstyle\max_y \{ (r+\epsilon_r)^\top y | Ay=0, \mathbf{1}^\top y = 1, y \ge 0, (c+\epsilon_c)^\top y \le d \}.
    \label{eq:y2}
\end{align}
where  
	$0 \le r \le \mathbf{1}$, $0 \le c \le \mathbf{1}$, $\epsilon_r \ge 0$, and $\epsilon_c \ge 0$ hold element wise. 
\begin{lemma} 
	Assuming the domains of \eqref{eq:y1} and \eqref{eq:y2} are not empty, let $y_1$ and $y_2$ be  solutions for each of the problems, respectively.  
	If, for some constant $\alpha > 0$ and $\beta > 0$, there exist $y_0 \in \{y | Ay = 0, \mathbf{1}^\top y = 1, y \ge 0, (c+\epsilon_c)^\top y \le d \}$ such that $r^\top (y_1-y_0) = \alpha >0$ and $c^\top (y_1-y_0) = \beta >0$, then
$	r^\top (y_1 - y_2) \le \frac{2\alpha}{\beta} \| \epsilon_c \|_1 + \|\epsilon_r\|_1.$
	\label{lemma:pseudoregret}
\end{lemma}
%
\begin{proof}
	Let \[y_3 = \arg\max_y \{ r^\top y | Ay=0, \mathbf{1}^\top y = 1, y \ge 0,
    (c+\epsilon_c)^\top y \le d \}.\]
	We first find the upper bound of $r^\top (y_1- y_3)$ where we note that $y_3$ and $y_1$ are the solutions of same linear program over different domains. Since the domain of $y_3$ is smaller than $y_1$, we know that $r^\top (y_1- y_3) \ge 0$. 
	First, consider the trivial case that $y_1$ satisfies $(c+ \epsilon_c)^\top
    y_1 \le d$. In this case, $y_1 = y_3$ and $r^\top (y_1- y_3) = 0$. Now we
    only consider the case such that $(c+ \epsilon_c)^\top y_1 > d$. Note that
    $(c+ \epsilon_c)^\top y_0 \le d$. Hence, there exists a $\gamma \in [0,1)$
        such that $y_4 = y_0 + \gamma (y_1-y_0)$ and $(c+\epsilon_c)^\top y_4 =
        d$---i.e., $\gamma = (d - (c+\epsilon_c)^\top y_0)/((c+\epsilon_c)^\top
        (y_1-y_0))$. Further, we have \[y_1-y_4=y_1 - y_0 - \gamma(y_1-y_0)=(1-\gamma)(y_1-y_0),
        \] so that $c^\top  (y_1-y_4)= (1-\gamma) \beta >0$
	and 
	\begin{align}
	c^\top (y_1-y_4) & = (c+\epsilon_c)^\top (y_1-y_4) - \epsilon_c^\top
    (y_1-y_4)\\
    &= c^\top y_1 + \epsilon_c^\top  y_1 -d - \epsilon_c^\top (y_1-y_4) \\
	& \le d-d + \epsilon_c^\top y_1 -  \epsilon_c^\top (y_1-y_4)\\
    &\le \| \epsilon_c \|_1 \|y_1\|_\infty + \| \epsilon_c \|_1 \| y_1-y_4
    \|_\infty\\
    &= 2\|\epsilon_c\|_1
	\end{align}
Combining this bound with
\begin{equation}
\frac{r^\top (y_1-y_4)}{c^\top  (y_1-y_4)} =  \frac{r^\top (y_1-y_0)}{c^\top
(y_1-y_0)} = \frac{\alpha}{\beta},
\end{equation}
we have that \[\textstyle 0 < r^\top  (y_1-y_4) \le 2 \frac{\alpha}{\beta}
\|\epsilon_c\|_1.\] Since the domain for each of these problems is convex, we
know that  \[y_4 \in \{y | Ay=0, \mathbf{1}^\top  y = 1, y\ge 0,
(c+\epsilon_c)^\top y \le d \}.\] Due to optimality, $r^\top  y_3 \ge r^\top
y_4$ so that \[\textstyle r^\top  (y_1-y_3) \le 2 \frac{\alpha}{\beta} \|\epsilon_c\|_1.\]

We leverage the bounud on  $r^\top (y_1 - y_3)$ to obtain a   bound on $r^\top
(y_3 - y_2)$. Note that $y_3$ and $y_2$ are the solutions of two linear programs
with different objectives but the same domain. According to optimality of the
solutions, we know that $r^\top  y_3 \ge r^\top  y_2$ and $(r+ \epsilon_r)^\top
y_2 \ge (r + \epsilon_r)^\top  y_3$. Combining these facts, we have that
	\begin{equation}
	0 \le r^\top  (y_3 - y_2) \le \epsilon_r^\top  (y_2 - y_3) \le \|\epsilon_r\|_1 \|y_2 - y_3\|_\infty \le \|\epsilon_r\|_1
	\end{equation}
	
	Now, combining the bounds on $r^\top  (y_3 - y_2)$ and $r^\top  (y_1 - y_3)$, we have that
	\begin{equation}
	\textstyle r^\top  (y_1 - y_2) = r^\top  (y_1 - y_3) + r^\top  (y_3 - y_2)
    \le \frac{2\alpha}{\beta} \| \epsilon_c \|_1 + \|\epsilon_r\|_1. 
	\end{equation}
\end{proof}

We can use the preceding lemma to get a  bound on the pseudo-regret. 

\begin{proposition} \label{thm:pseudoregret}
	Denote $\mathcal{Y} = \{y | (I_o - P)y = 0, \mathbf{1}^\top y = 1, y \ge 0 \}$. If there exists $y_0 \in \mathcal{Y}$ such that $\bar{r}^\top (\bar{y}-y_0) = \alpha >0, \bar{c}^\top(\bar{y}-y_0) = \beta >0$, then with probability at least $1-\delta$,
	\begin{equation}
\textstyle	\bar{r}^\top (\bar{y} - \tilde{y}_k) \le 2 (\frac{2\alpha m}{\beta} + 1)  \sum_{s, a} \big(\frac{\log (S A (m+1) \pi^2 t_k^3 / 3 \delta)}{2 \max \{1, N_k(s,a)\}}\big)^{1/2}.
	\end{equation}
\end{proposition}
\begin{proof}
	By definition \[\bar{y} = \arg\max_y \{ \bar{r} ^\top y | (I_o - P) y=0,
        \mathbf{1}^\top y = 1, y \ge 0, \bar{c}_i^\top y \le d_i, i \in [m]\}\]
        and \[\tilde{y}_k = \arg\max_y \{ \tilde{r}_k^\top y | (I_o - P) y=0,
            \mathbf{1}^\top y = 1, y \ge 0, \tilde{c}_{i,k}^\top y \le d_i, i
        \in [m]\}.\] Define a sequence of subproblems by adding the confidence value to one additional constraint 
        at a time as follows:
	\begin{equation*}
	\begin{aligned}
	y^{(1)} & = \arg\max_y \{ \bar{r} ^\top y | (I_o - P) y=0, \mathbf{1}^\top y
= 1, y \ge 0, \tilde{c}_{1,k}^\top y \le d_1, \bar{c}_i^\top y \le d_i, i \in
\{2, \dots, m\}\} \\
	y^{(2)} & = \arg\max_y \{ \bar{r} ^\top y | (I_o - P) y=0, \mathbf{1}^\top y
= 1, y \ge 0, \tilde{c}_{1,k}^\top y \le d_1, \tilde{c}_{2,k}^\top y \le d_2,
\bar{c}_i^\top y \le d_i, i \in \{3, \dots, m\}\}  \\
	& \ \ \vdots \\
	y^{(m)} & = \arg\max_y \{ \bar{r} ^\top y | (I_o - P) y=0, \mathbf{1}^\top y
    = 1, y \ge 0, \tilde{c}_{i,k}^\top y \le d_i,  i \in \{1, \dots, m\}\}
	\end{aligned}
	\end{equation*}
	
	Using the same proof technique as for that of
    Lemma~\ref{lemma:pseudoregret}, we obtain the bounds for each of the
    subproblems
    \[\bar{r}^\top (\bar{y} - y^{(1)}), \bar{r}^\top (y^{(1)} - y^{(2)}), \dots,
        \bar{r}^\top (y^{(m-1)} - y^{(m)}), \bar{r}^\top (y^{(m)} -
    \tilde{y}_{k}).\] 
    Combining each of the bounds and the fact that \[\textstyle
        |\tilde{r}_k(s,a) - \bar{r}(s,a)| \le 2 \sqrt{\frac{\log (S A (m+1) \pi^2
        t_k^3 / 3 \delta)}{2 \max \{1, N_k(s,a)\}}},\]
        and \[\textstyle |\tilde{c}_{i,k}(s,a) - \bar{c}_i(s,a)| \le 2
            \sqrt{\frac{\log (S A (m+1) \pi^2 t_k^3 / 3 \delta)}{2 \max \{1,
            N_k(s,a)\}}},\] 
            we have that
	\begin{align*}
\textstyle	\bar{r}^\top (\bar{y} - \tilde{y}_{k}) \textstyle = \bar{r}^\top (\bar{y} - y^{(1)}) + \dots + \bar{r}^\top (y^{(m)} - \tilde{y}_{k}) &\le\textstyle m \frac{2\alpha}{\beta} \| \tilde{c}_k - \bar{c} \|_1 + \| \tilde{r}_k - \bar{r} \|_1 \\
	&\textstyle\le 2 (\frac{2\alpha m}{\beta} + 1) \sum_{s,a} \sqrt{\frac{\log (S A (m+1) \pi^2 t_k^3 / 3 \delta)}{2 \max \{1, N_k(s,a)\}}}
	\end{align*}
	which completes the proof.
\end{proof}


Note that according to Proposition~\ref{thm:pseudoregret}, with probability at least $1-\delta$,
the per-step pseudo-regret of executing policy $\tilde{\pi}_k$ depends on the confidence intervals of reward and costs of all state-action pairs. This is intuitive since in order for the policy $\tilde{\pi}_k$ to be close to the optimal policy $\bar{\pi}$, we need to have good approximations of the reward and costs for all state-action pairs. To ensure this, we need to constantly explore the CMDP so that $N_k(s,a)$ is not `too small' for any state-action pair. Since the Markov chain resulting from the baseline policy is irreducible and aperiodic, the steady state occupation measure $y_0(s,a)$ corresponding to the baseline policy $\pi_0(a|s)$ has the property that $y_0(s,a) > 0, \forall s, a$.
Due to this universal exploration demand, we execute the baseline policy $\pi_0$ for a constant number of times in each linear increasing episode in the {\CUCRL} algorithm.

To have a upper bound on the regret derived in Proposition~\ref{thm:pseudoregret}, we need to have a lower bounds on $N_k(s,a)$. Given our assumptions on the baseline policy as discussed above, define $\rho>0$ such that $y_0(s,a) \ge \rho > 0$  for all state-action pairs $(s,a)\in \mathcal{S}\times\mathcal{A}$.
The following lemma gives a lower bound on the number of times each state-action pair is visited in episode $k$.

\begin{lemma}  \label{lemma:visitcount}
Given a fixed total number of episodes $K$,	with probability at least
$1-\delta$, for every state-action pair $(s,a)$ and episode $k\in [K]$,
	\begin{equation}
\textstyle	N_k(s,a) \ge (k-1) \rho h - (k-1) \big(72\xi \rho h \log\big(\frac{\varphi \cdot SAK}{\delta}\big)\big)^{1/2}
	\end{equation}
	where $\xi$ the mixing time of the Markov chain induced by policy $\pi_0$, $\rho>0$ is such that $y_0(s,a) \ge \rho > 0$  for all state-action pairs $(s,a)\in \mathcal{S}\times\mathcal{A}$, and 
	$\varphi = \sum _{s,a} \frac{y'(s,a)^2}{y_0(s,a)} $, where $y'$ is the initial state action distribution and $y_o$ is the steady state action distribution under the baseline policy.
\end{lemma}
\begin{proof}
	Consider the exploration phase (when the baseline policy $\pi_0$ is
    executed) of the $k$-th episode in Algorithm~\ref{alg:safeUCRL}. For a given episode $\ell$ and for a fixed
    state-action pair $(s,a)$, let $X_{\ell,1}, \dots, X_{\ell,h}$ be the indicator variables
    of whether state-action pair $(s,a)$ has be selected at each step within the episode $\ell$. Let $Y_\ell =
    \sum_{i=1}^hX_{\ell,i}$ and thus $\mathbb{E}[Y_\ell] = y_0(s,a)h$. Applying the
    Chernoff-Hoeffding bound in \citep[Theorem 3]{chung2012chernoff}, gives 
	\begin{equation}
	\textstyle\Pr \{ \mathbb{E}[Y_\ell] - Y_\ell \ge \epsilon y_0(s,a) h \} \le \varphi  \cdot
    \exp\left(-\frac{\epsilon^2 y_0(s,a) h}{72 \xi}\right).
	\end{equation}
	Setting
	\begin{equation}
\textstyle	\epsilon = \sqrt{\frac{72\xi}{y_0(s,a)h} \log(\frac{\varphi 
SAK}{\delta})},
	\end{equation}
    the above bound becomes 	\begin{equation}
\textstyle	\Pr \left\{ Y_\ell \le y_0(s,a)h - \sqrt{72\xi y_0(s,a)h \log(\frac{\varphi SAK}{\delta})} \right\} \le \frac{\delta}{SAK}
	\end{equation}
	Using the assumption that $y_0(s,a) \ge \rho > 0, \forall (s,a)$, the union
    bound over all state-action pairs $(s,a)$ and episodes $k\in [K]$ is given by
	\begin{equation}
\textstyle	\Pr \left\{ \bigcup_{(s,a), \ell} \left\{ Y_\ell \le \rho h - \sqrt{72\xi
    \rho h \log(\frac{\varphi SAK}{\delta})} \right\} \right\} \le \sum_{\ell=1}^K\sum_{s,a} \frac{\delta}{SAK} = \delta
	\end{equation}
	Now, we note that
	\[\textstyle\left\{\sum_{\ell=1}^{k-1} Y_\ell\leq (k-1)\left(\rho h-\sqrt{72\xi
    \rho h \log(\frac{\varphi SAK}{\delta})}\right) \right\}\subset \bigcup_{\ell=1}^{k-1} \left\{ Y_\ell \le \rho h - \sqrt{72\xi
    \rho h \log(\frac{\varphi SAK}{\delta})} \right\} \]
    and $N_k(s,a) \ge  \sum_{\ell=1}^{k-1}Y_\ell$ since in each episode $\tilde{\pi}_\ell$ is executed $h\ell$ times after the baseline policy so that $N_\ell(s,a)$ may be larger. Hence,
	\begin{equation}
\textstyle	N_k(s,a) \ge (k-1) \rho h - (k-1) \sqrt{72\xi \rho h \log(\frac{\varphi SAK}{\delta})}
	\end{equation}
	holds with probability at least $1-\delta$.

\end{proof}

%

Combining Proposition~\ref{thm:pseudoregret} and Lemma~\ref{lemma:visitcount} and summing over $K$ episodes, we  obtain the total regret bound for {\CUCRL}.  
\begin{theorem}
Suppose that $\delta \le \varphi SAK \exp(-\frac{\rho h}{288 \xi})$.	Under the assumptions of Proposition~\ref{thm:pseudoregret}, with probability at least $1-\delta$, {\CUCRL} has total pseudo-regret 
	$\Delta(T) = O(T^{\frac{3}{4}}\sqrt{\log(T/\delta)})$.
	\label{thm:regret}
\end{theorem}
\begin{proof}
	According to Proposition~\ref{thm:pseudoregret}, the total regret of $K$ episodes is 
	\begin{align*}
	\textstyle	\sum_{k=1}^K \Delta_k &\textstyle = \sum_{k=1}^K h \bar{r}^\top  (\bar{y} - y_0) + (k-1)h \bar{r}^\top  (\bar{y} - \tilde{y}_k)\\
	&=\textstyle h K \bar{r}^\top  (\bar{y} - y_0) + h \sum_{k=2}^K (k-1) \bar{r}^\top  (\bar{y} - \tilde{y}_k) \\
		& \textstyle\le 2hK + 2 (\frac{2\alpha m}{\beta} + 1) h \sum_{k=2}^K (k-1) \sum_{s, a} \sqrt{\frac{\log (S A (m+1) \pi^2 t_k^3 / 3 \delta)}{2 \max \{1, N_k(s,a)\}}}
	\end{align*}
	Let 
\[\textstyle\zeta=\rho h - \left(72\xi \rho h \log(\frac{\varphi  SAK}{\delta})\right)^{1/2}.\]  
Since $\delta \le \varphi SAK \exp(-\frac{\rho h}{288 \xi})$, we have that  \[\left(72\xi \rho h \log(\frac{\varphi  SAK}{\delta})\right)^{1/2} \ge \frac{1}{2} \rho h\] so that  
	$\zeta \le \frac{1}{2} \rho h$.
	
	Combining this with Lemma \ref{lemma:visitcount}, we have that
	\begin{align*}
		\Delta(T)  =\textstyle \sum_{k=1}^K \Delta_k  &\textstyle\le 2hK + 2 (\frac{2\alpha m}{\beta} + 1) h \sum_{k=2}^K (k-1) \sum_{s, a} \left(\frac{\log (S A (m+1) \pi^2 t_k^3 / 3 \delta)}{2 \max \{1, N_k(s,a)\}}\right)^{1/2} \\
		& \textstyle\le 2hK + 2 (\frac{2\alpha m}{\beta} + 1) hSA \sum_{k=2}^K (k-1) \left(\frac{\log (SA (m+1) \pi^2 t_k^3 / 3 \delta)}{2 (k-1)\zeta}\right)^{1/2} \\
		& \textstyle\le 2hK + 2 (\frac{2\alpha m}{\beta} + 1) hSA \sqrt{\log (SA (m+1) \pi^2 T^3 / 3 \delta)} \sum_{k=1}^{K-1} \sqrt{\frac{k}{2\zeta}} \\
		& \textstyle= 2hK + 2 (\frac{2\alpha m}{\beta} + 1) hSA \sqrt{\frac{\log (SA (m+1) \pi^2 T^3 / 3 \delta)}{2\zeta}} \sum_{k=1}^{K-1} \sqrt{k} \\
		& \textstyle\le 2hK + 2 (\frac{2\alpha m}{\beta} + 1) hSA \sqrt{\frac{\log (SA (m+1) \pi^2 T^3 / 3 \delta)}{\rho h}} \sum_{k=1}^{K-1} \sqrt{k} \\
		& \textstyle\le 2hK + 2 (\frac{2\alpha m}{\beta} + 1) hSA \sqrt{\frac{\log (SA (m+1) \pi^2 T^3 / 3 \delta)}{\rho h}} (K-1) \left(\frac{K}{2}\right)^{1/2} \\
		&\textstyle = O(K) + O(K\sqrt{K\log(T/\delta)}) \\
		& \le O(T^{\frac{3}{4}}\sqrt{\log(T/\delta)})
	\end{align*}
	where the second to last inequality follows from Jensen's inequality and the final step follows from $T = \sum_{k=1}^K kh = \frac{K(K-1)}{2} h$ so that $K < (2T/h)^{1/2}$.

\end{proof}

\noindent\textbf{Remark}. \emph{Adding constants related to the dimension of the CMDP, we have the regret bound
\begin{equation}
\Delta(T) \le O(mSA T^{\frac{3}{4}}\sqrt{\log(mSAT/\delta)}).
\end{equation}
}



\subsection{Specializing to the Constrained Multi-Armed Bandit Setting}

Constrained Multi-Armed Bandits (CMABs) can be viewed as a special case of CMDPs, where there is only one state, $S=1$ and the transition kernel is trivially staying in that state with all actions. The policy in a CMAB is a probabilistic distribution over actions/arms $y(a)$ and the goal  is to solve the following linear program:
\begin{equation}
\max_{y} \{ \bar{r}^\top y \ | \ \mathbf{1}^\top y = 1, y \ge 0,  \bar{c}^\top y \le d  \}.
\end{equation}
Similarly, the per-step pseudo-regret is defined as $\bar{r}^\top (\bar{y} - \tilde{y}_k)$ where $\bar{y}$ is the optimal randomized policy and $\tilde{y}_k$ is the policy execute in episode $k$ of {\CUCRL}. Running {\CUCRL} with $S=1$, the following corollaries hold.
\begin{corollary} 
In CMABs, {\CUCRL} is $\delta$-safe.
\end{corollary}

\begin{corollary} 
In CMABs, Under the assumptions of Proposition~\ref{thm:pseudoregret}, with probability at least $1-\delta$, {\CUCRL} has total pseudo-regret $\Delta(T) = O(T^{\frac{3}{4}}\sqrt{\log(T/\delta)})$.
\end{corollary}
The proofs of the above two corollaries follow directly from the corresponding results in the preceding section. 

\section{Experiments}
\label{sec:experiment}

The goal of this section is to explore a few illustrative examples which highlight different features of our approach.

\subsection{Two Armed Bandit with per Round Budget Constraints}

\begin{figure}[t]
	\centering
	\subfigure[ ]{\label{subfig:cbandit}\includegraphics[width=0.31\linewidth]{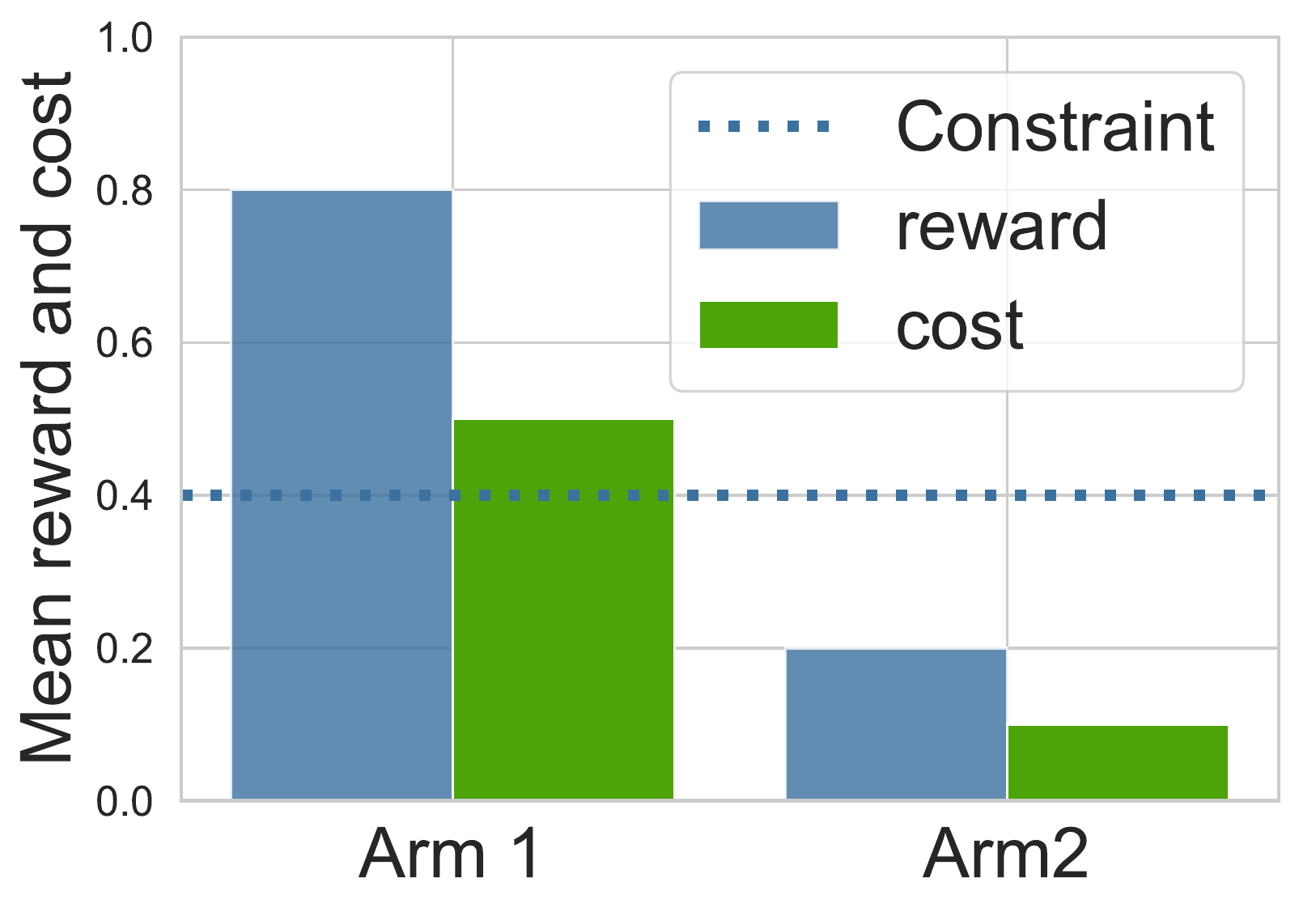}} 
	\subfigure[]{\label{subfig:pull_count}\includegraphics[width=0.3\linewidth]{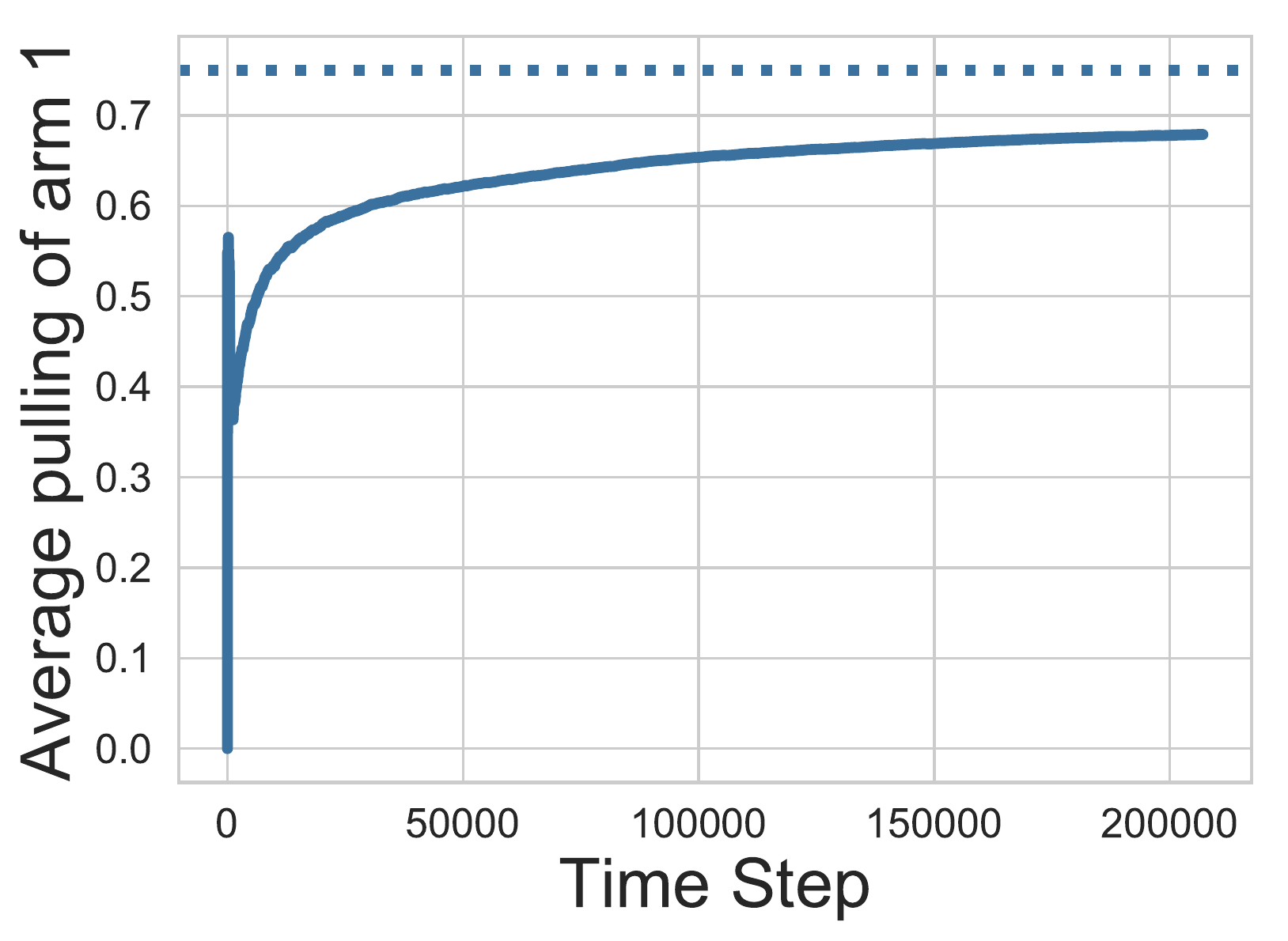}}
	\subfigure[]{\label{subfig:regret_bandit}\includegraphics[width=0.31\linewidth]{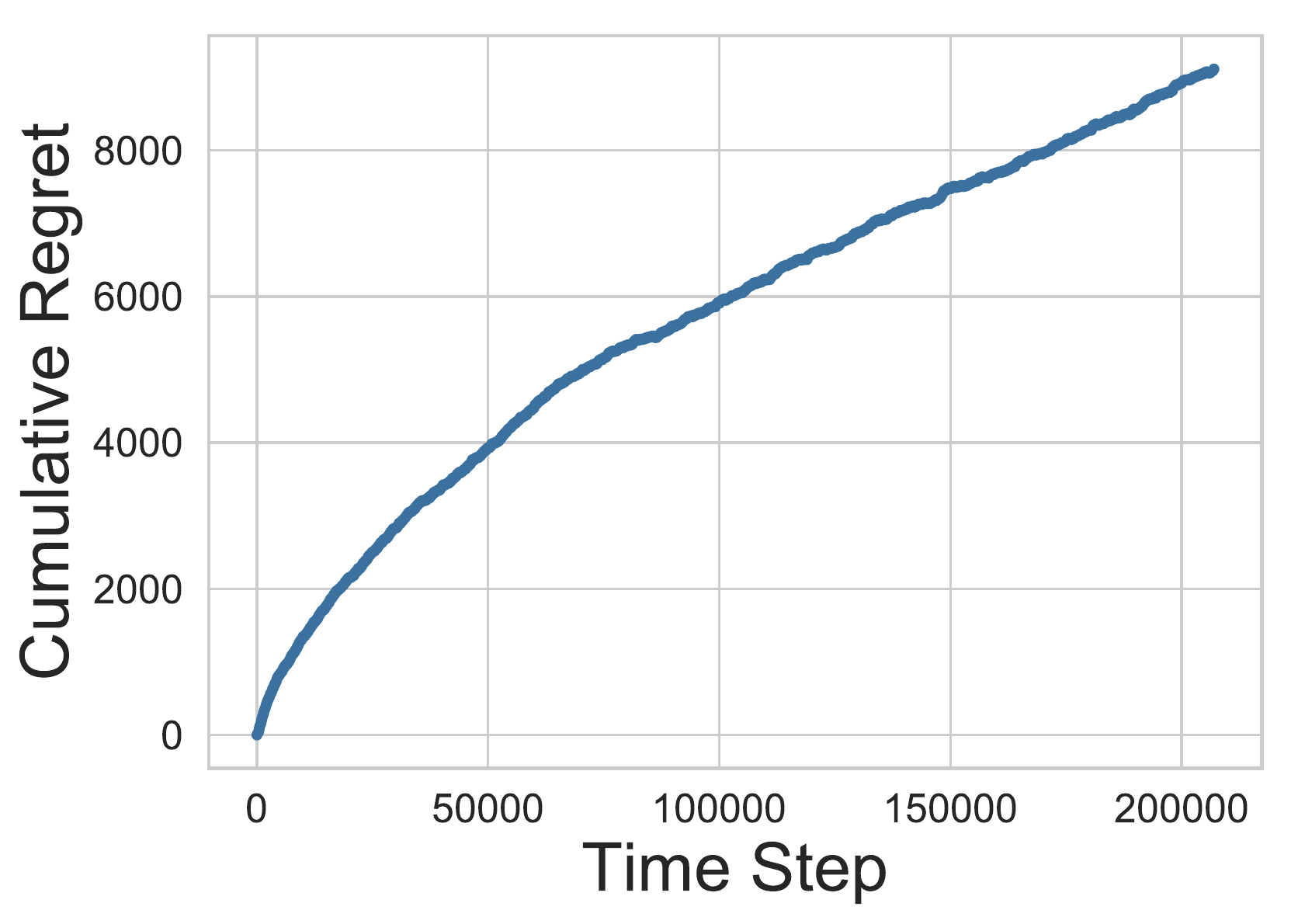}}
	\caption{Two armed bandit with per-round budget constraint: (a)  mean reward and cost of each arm as well as the per-round constraint; (b)  average number of times arm one is pulled; (c) the cumulative regret of {\CUCRL}.}
	\label{fig:cbandit}
\end{figure}

We first consider a simple two arms bandit example. As stated before, the CMDP reduces to a constrained multi-armed bandit problem when $|\mathcal{S}|=1$.  The reward and cost of each arm are unknown and stochastic. In our simulation, the reward and cost is draw from a binomial distribution, with the mean shown in Figure \ref{subfig:cbandit}.  Even though arm one has a better reward, we cannot pull arm one all the time since the constraint is set to be less than the mean cost of arm one. The optimal policy is to pull arm one with probability $0.75$ and arm two with probability $0.25$. The baseline policy we use to start exploration is pulling the two arms uniformly at random. Figure \ref{subfig:pull_count} and \ref{subfig:regret_bandit} show the average number of times arm one is pulled and the cumulative regret of {\CUCRL}, respectively.  The average pull count of arm one never exceeds $0.75$.
\begin{algorithm}[t]
	\SetAlgoLined
	\KwIn{safety parameter $\delta \in (0,1)$, baseline policy $\pi_0(a|s)$, episode length $h$, risk sensitive parameter $\lambda$.}
	{\textbf{Initialization:} set $t=1$, observe the initial state $s_1$ }
	
	\For{$\mathrm{episodes}$ $k=1, 2, \dots, K$}{
		$t_k = t$ \tcp*{initialize start time of episode $k$ }
		\While{$t\leq t_k + h$\tcp*{Execute baseline policy $h$ times for exploration}}{  
			
			Draw action $a_t \sim \pi_0(\cdot|s_t)$ 
			
			Observe reward $r_t$, costs $c_{i,t}$, and the next state $s_{t+1}$ 
			
			$t\gets t+1$
			
			
		}

		$N_k(s,a) = \textstyle\sum_{t' = 1}^{t} \mathbf{1}(s_{t'} = a, a_{t'} = a), \ \forall (s,a)\in \mathcal{S}\times\mathcal{A}$ \tcp*{set the state-action count}
		
		$R_k(s, a) =  \textstyle\sum_{t' = 1}^{t} (r_{t'} - \lambda^\top c_{t'}) \mathbf{1}(s_{t'} = a, a_{t'} = a) $\tcp*{cumulative reward cost trade-off}
		
		$\hat{r}_k(s, a) = \textstyle\frac{R_k(s, a)}{\max \{1, N_k(s, a)\}}, \quad \tilde{r}_k(s,a) = \hat{r}_k(s, a) + \sqrt{\frac{7 \log (2SAt_k / \delta)}{2 \max \{ 1, N_k(s,a) \}}}$ \tcp*{compute estimates}
		
		$\tilde{y}_k\gets \arg\max \{ \tilde{r}_k^\top y | I_o y = Py, \mathbf{1}^\top y = 1, y \ge 0 \}$ 
		
		$\tilde{\pi}_k\gets \tilde{y}_k(s,a)/(\sum_{a\in \mathcal{A}}\tilde{y}_k(s,a))$ \tcp*{recover policy}
		

		\While{$t\leq t_k + kh$ \tcp*{Execute $\tilde{\pi}_k$ policy $(k-1)h$ times}}{
			
			Draw action $a_t \sim \tilde{\pi}_k(\cdot|s_t)$ 
			
			Observe reward $r_t$, costs $c_{i,t}$, and the next state $s_{t+1}$
			
			$t\gets t+1$
			
		}
	}
	\caption{risk-sensitive UCRL2 ({\RSUCRL}) algorithm}
	\label{alg:rsUCRL}
\end{algorithm}
\subsection{Three State CMDP}
 To  demonstrate the performance of {\CUCRL}, we consider a simple three state CMDP.  As show in Figure \ref{subfig:cmdp}, the CMDP we consider has three states and two actions. An agent can take either a \emph{risky} exploratory action in which the navigate to another state or they can take the \emph{safe} action and remain in the current state.
 There is no reward or cost for staying in the current state but there will be a stochastic reward and cost if the agent navigates. In the simulation, the reward and cost of each state-action pair are each draw from a binomial distribution, with the means defined in the labels on edges in Figure~\ref{subfig:cmdp}.  Obviously, without this constraint, the optimal policy is to navigate in each of the states. In this problem, we consider the constraint that in expectation, the average cost should be less than $0.2$. This constraint prevents the agents from continuously navigating between the three states. In particular, as shown in Figure~\ref{subfig:opt_policy}, the constrained optimal policy is a randomized policy that has positive probability on the safe action in each state. The relatively conservative baseline policy we use in {\CUCRL} for exploration is staying in the current state with probability $0.8$ and navigate to the next state with probability $0.2$. 
\begin{figure}[t]
	\centering
	\subfigure[]{\label{subfig:cmdp}\includegraphics[width=0.3\linewidth]{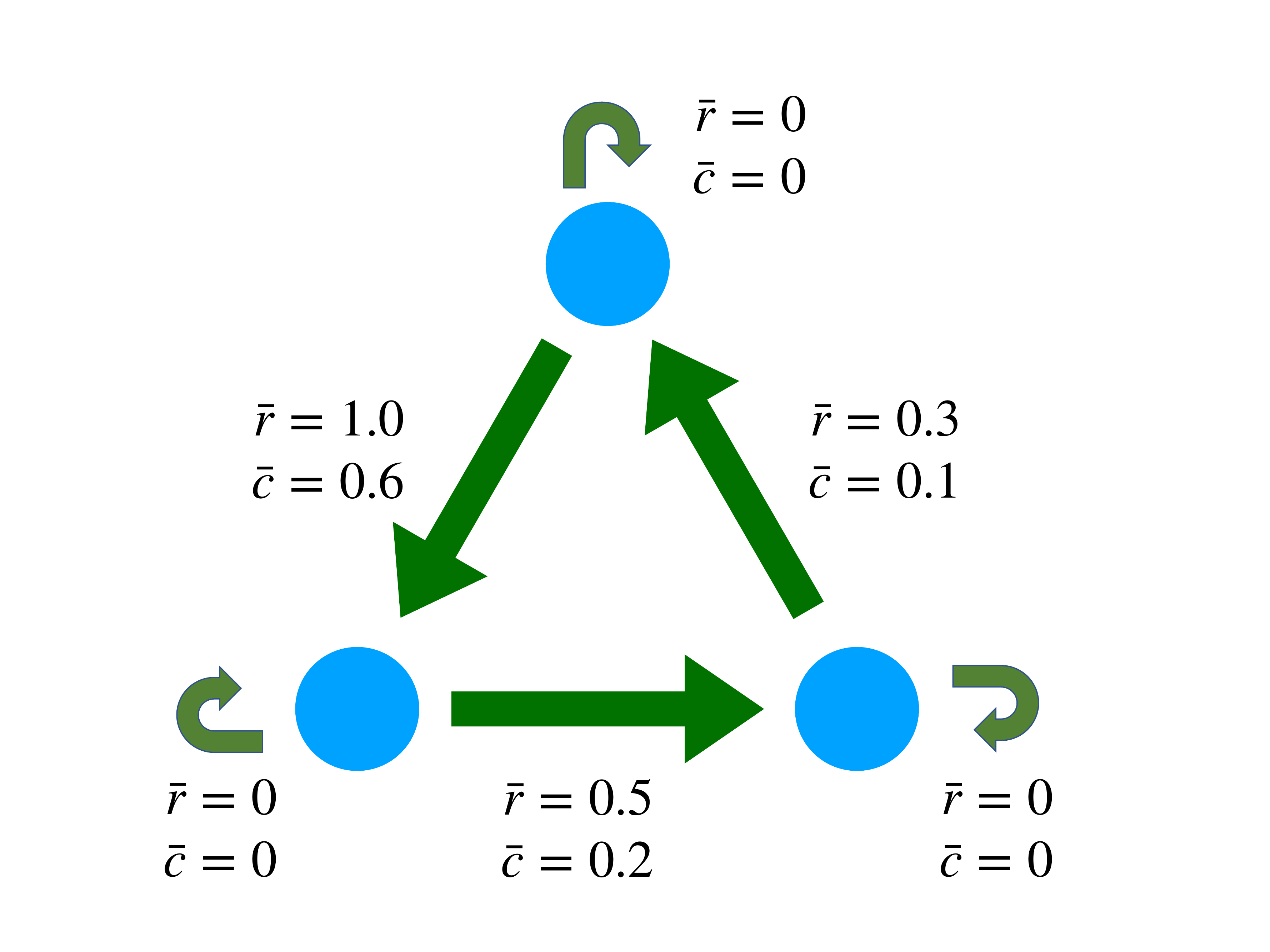}} 
	\subfigure[]{\label{subfig:opt_policy}\includegraphics[width=0.4\linewidth]{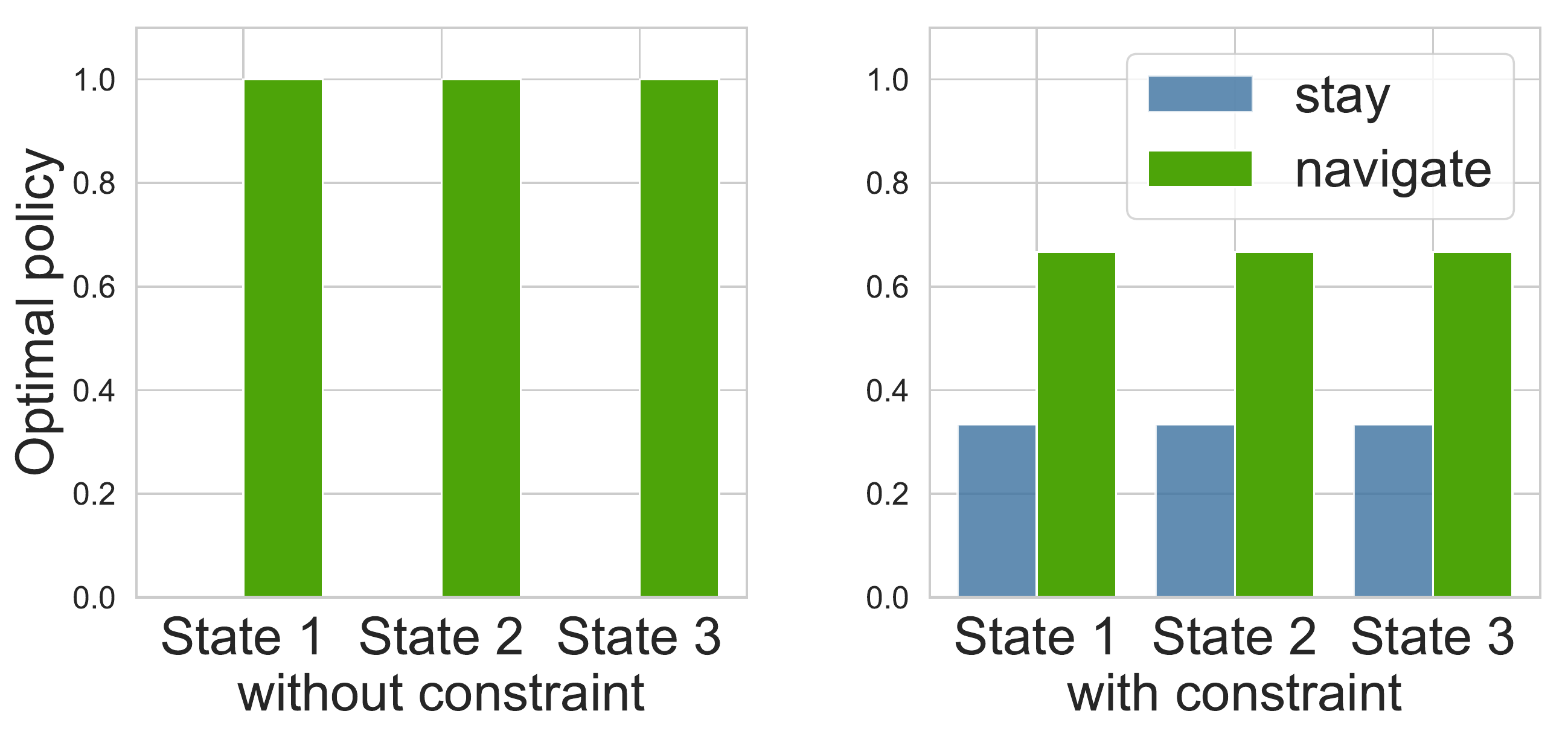}}
	\subfigure[]{\label{subfig:lambda}\includegraphics[width=0.19\linewidth]{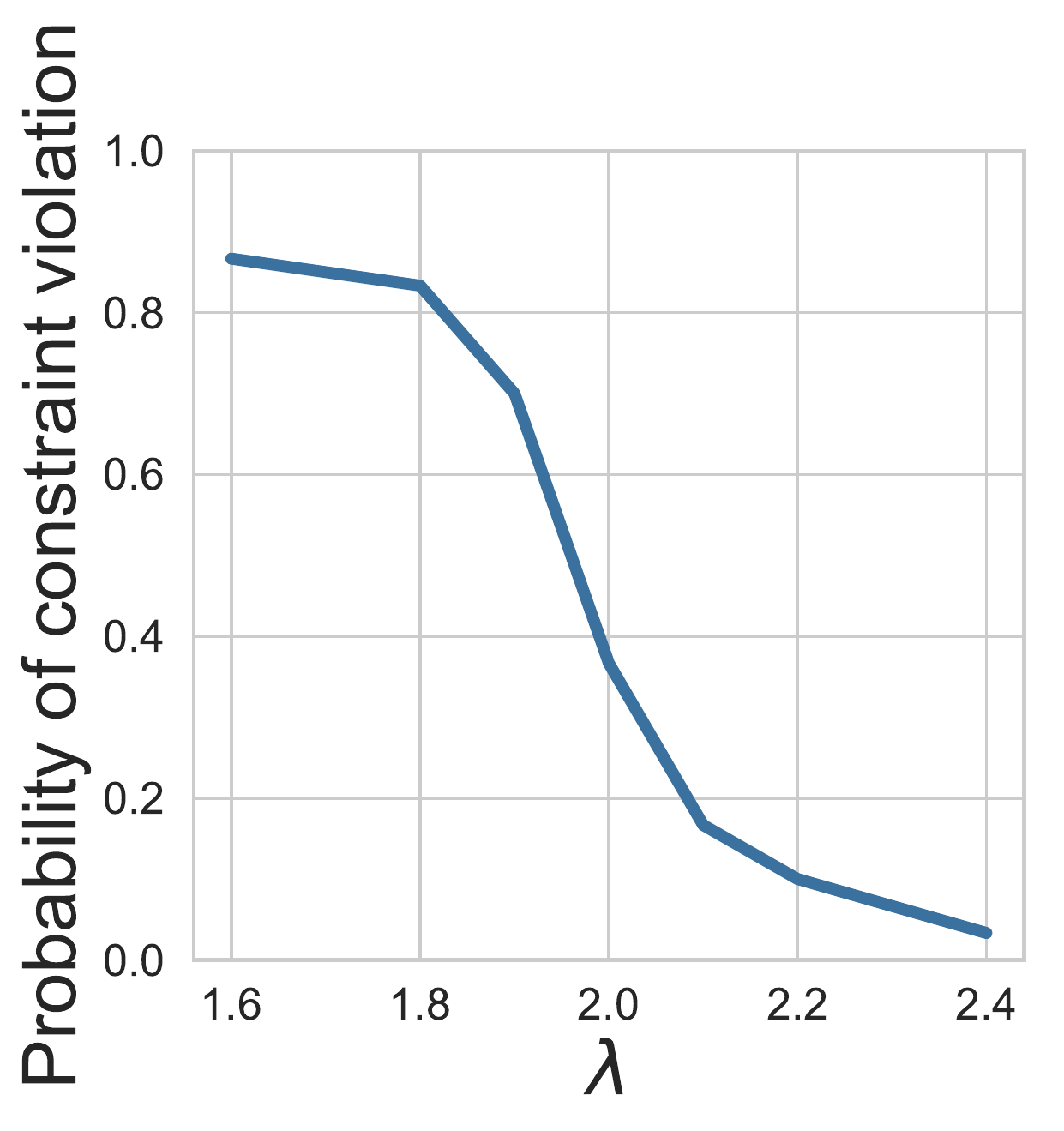}}
	\caption{Simple CMDP. (a) CMDP structure; (b) optimal policy computed with the true mean reward and mean cost, with and without the constraint on cost, $d=0.2$; (c) probability of constraint violation in 30 training episodes by risk-sensitive {\UCRL} ({\RSUCRL}). }
	\label{fig:cmdp}
\end{figure}
\begin{figure}[t]
	\centering
		\subfigure[]{\label{fig:regret_cost}\includegraphics[width=0.7\linewidth]{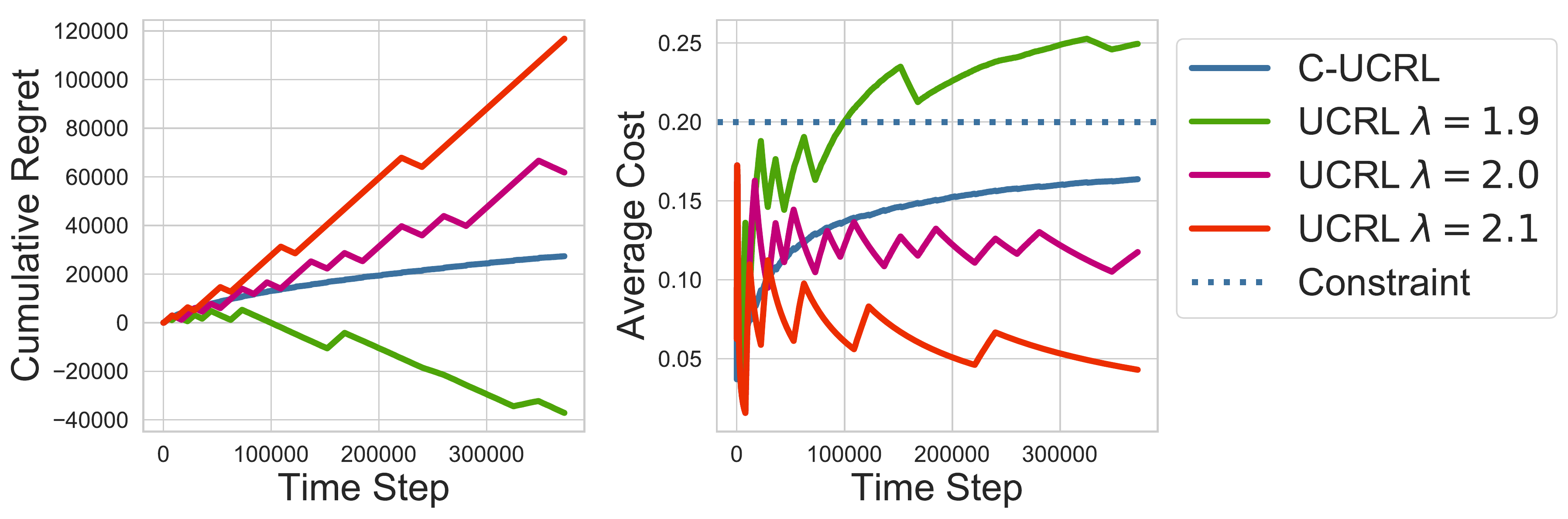}}
		
		\subfigure[]{\label{fig:policy}\includegraphics[width=0.75\linewidth]{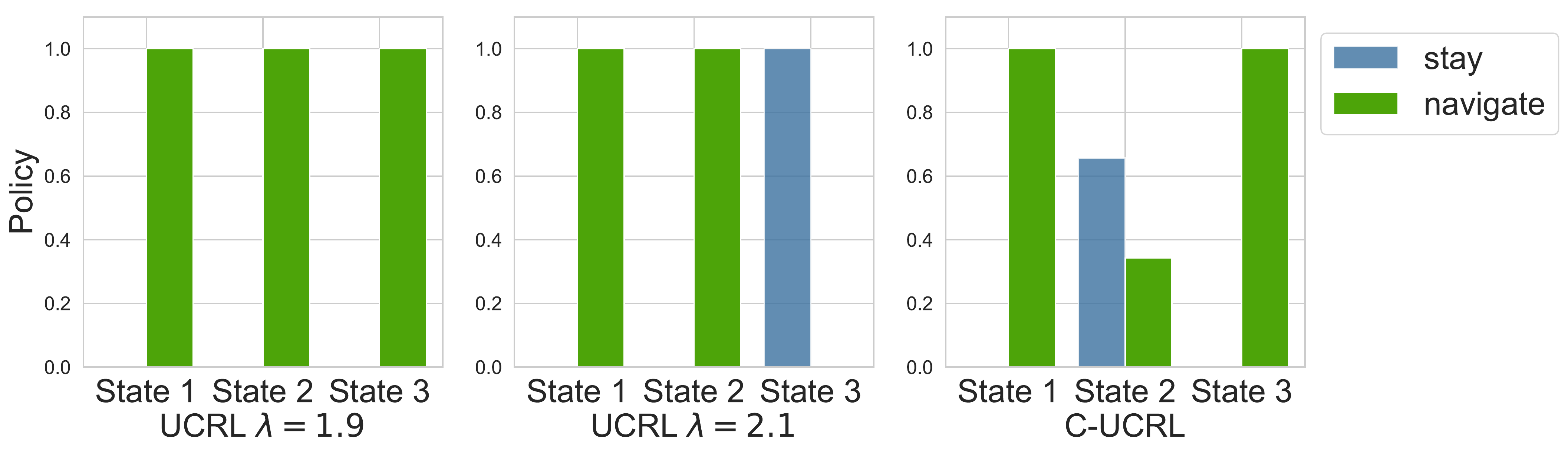}}
	\caption{{\CUCRL} vs.~{\RSUCRL}: (a) Cumulative regret and average cost for {\CUCRL} and risk sensitive {\UCRL}; (b) Policy learned by {\CUCRL} and  {\RSUCRL}.}
\end{figure}

We compare our approach with the {\UCRL} algorithm. However, {\UCRL} does not allow for constraints or multiple reward/cost criteria. Hence,  we leverage the idea of risk sensitive reinforcement learning \citep{leike2017ai}, where we treat a linear combination of reward and cost---i.e., $r - \lambda c$---as the reward for the {\UCRL} algorithm (Algorithm \ref{alg:rsUCRL}). The hyperparameter $\lambda$ represents the trade off between the reward and cost, the combination of which represents the reward in the classical implementation of {\UCRL}; we refer to risk-sensitive {\UCRL} by {\RSUCRL}. Figure \ref{subfig:lambda} shows the constraint violation probability in $30$ training episodes by {\RSUCRL} algorithm with different $\lambda$. Figure \ref{fig:regret_cost} shows the cumulative regret and average cost of the {\CUCRL} and {\RSUCRL} algorithms. As we can see,  when the cost value is underestimated ($\lambda = 1.9$), applying {\RSUCRL} directly leads to a `good' reward (i.e., the regret is negative as it gets more reward than the optimal randomized policy), yet the constraints are violated.  On the other hand, when the costs are overestimated ($\lambda = 2.1$), {\RSUCRL} is too conservative about the cost and, thus, receives high regret. We can observe that {\CUCRL} does not violate the constraint during learning though in this experiment, $\delta$ is set to be $0.1$, meaning that with probability at least $0.9$, the constraint will not be violated in all episodes.

The fundamental problem with {\RSUCRL} is that with only one criterion, the policy it learns will always be a deterministic policy, while in this CMDP, the optimal policy is randomized.  
Figure \ref{fig:policy} shows the policy learned by {\CUCRL} and {\RSUCRL}. When $\lambda = 1.9$, {\RSUCRL} learn the optimal policy as there is no constraint, which leads to constraint violation. When $\lambda=2.1$, the policy learned by {\RSUCRL} is to stay in one state forever. On the contrary, the policy learned by {\CUCRL} algorithm converges to the optimal randomized policy.

\subsection{Grid World with Safety Constraints}

Motivated by the goal of ensuring safety in reinforcement learning safety, we validate our algorithms using a 2D grid-world exploration problem \citep[2.24]{leike2017ai}. This example also represents a crude abstraction of rovers exploring the surface of Mars as described in \citep{wachi2018safe}. 

\begin{figure}[t]
	\centering
	\subfigure[]{\label{subfig:grid_world}\includegraphics[width=0.3\linewidth]{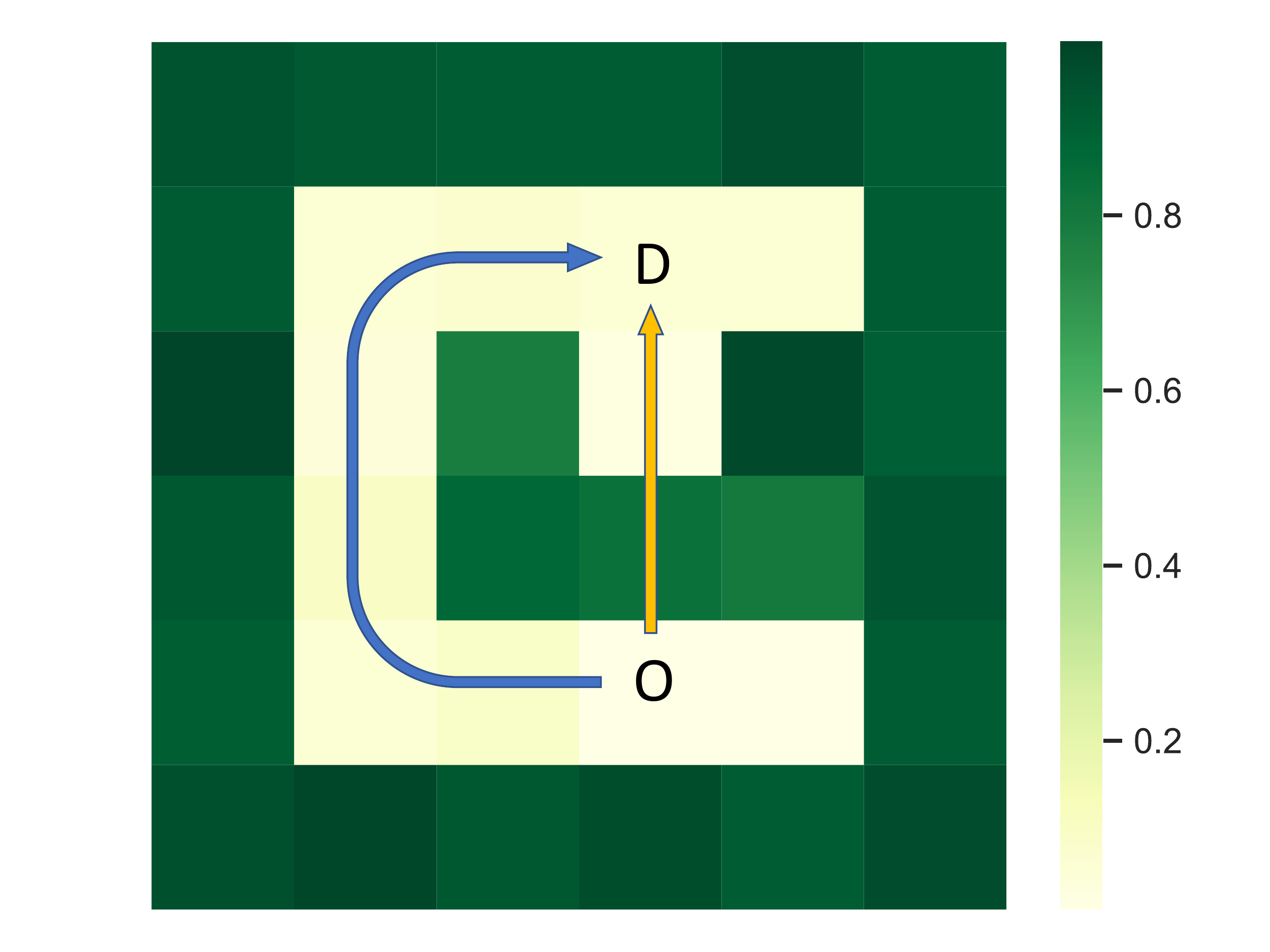}} 
	\subfigure[]{\label{subfig:grid_world_policy}\includegraphics[width=0.6\linewidth]{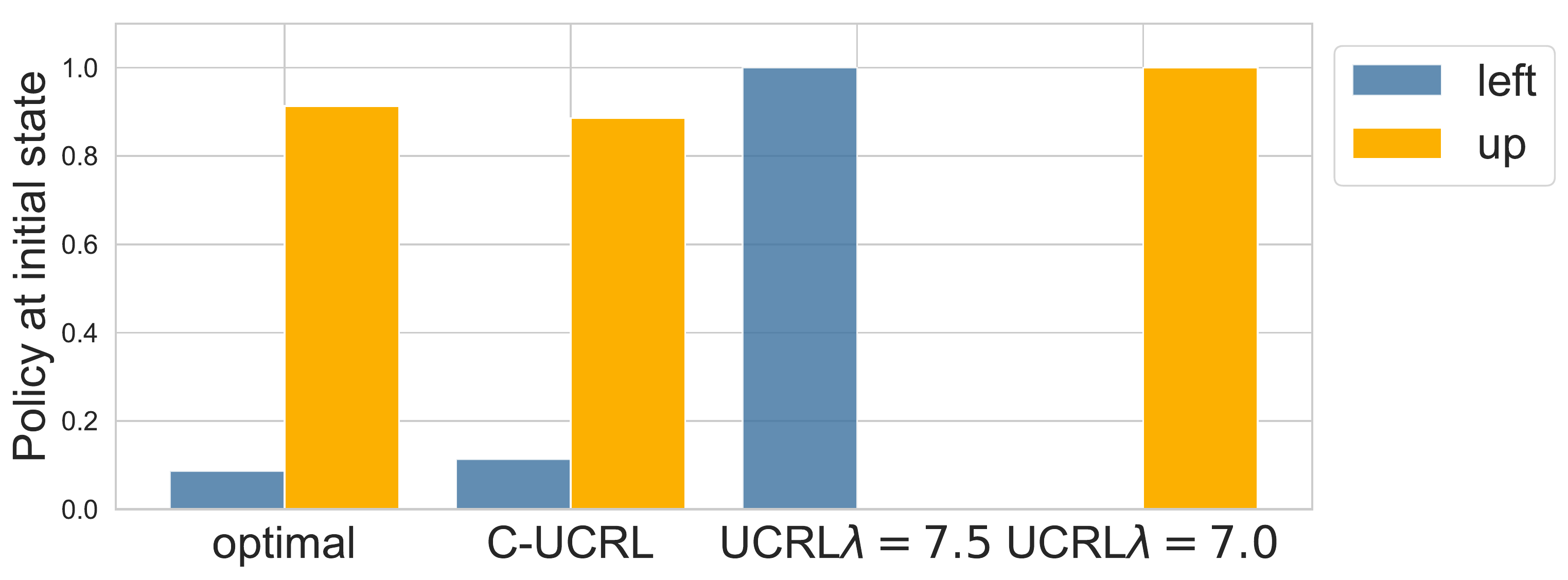}} 

	\caption{Grid World with Safety Constraints. (a) Grid world structure: the states with darker green color have larger mean cost, and `O' and `D' are the origin and destination states, respectively; (b) Policy learned by different algorithms: the blue column represents the probability of going `West' (choose blue route) and orange column represents the probability of going `North' (choose orange route).}
	\label{fig:grid_world}
\end{figure}

\begin{figure}[t]
	\centering
	\includegraphics[width=0.75\linewidth]{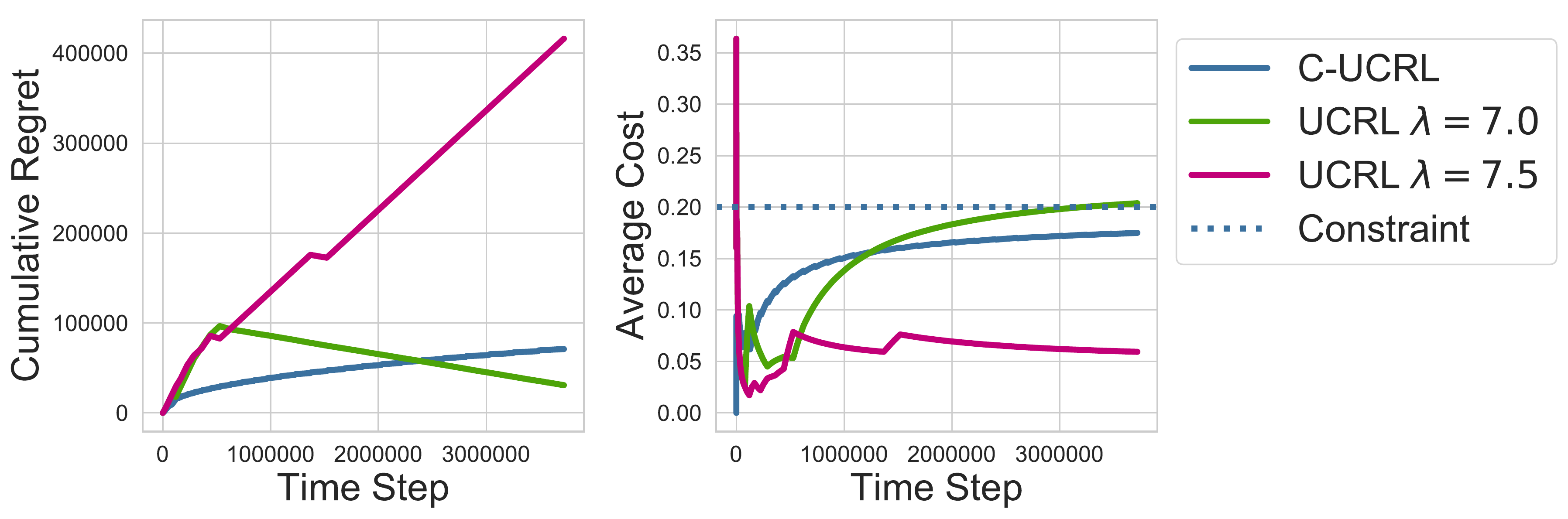}
	\caption{Cumulative regret and average cost of {\CUCRL} and  {\RSUCRL}.}
	\label{fig:grid_world_regret_cost}
\end{figure}

Figure \ref{subfig:grid_world} shows the CMDP structure. The green color in each state represents the mean cost of that state, and the darker the color, the higher the cost is. In the Mars exploration problem, those darker states are the states with large slope that the agents want to avoid. The constraint we enforce is the upper bound of the per-step probability of step into those state with large slope---i.e., the more risky or potentially unsafe states to explore. The agent starts from the origin state `O' and receives reward $1$ if it reaches the destination state `D' after which it returns to the origin. In the simulation, the cost of each state is draw from a binomial distribution, with the mean shown in the figure. At each time step, the agent can take action to move into any of its four neighboring states. Due to the stochastic environment, transitions are stochastic (i.e., even if the agent’s action is to go ``North”, the environment can send the vehicle with a small probability to ``East”).

Without  safety constraints, the optimal policy is obviously to always choose the orange route in Figure \ref{subfig:grid_world}. However, with constraints, as we can see in Figure \ref{subfig:grid_world_policy}, the optimal policy is a randomized policy that use both blue and orange routes with some probabilities. The relatively conservative baseline policy we use in {\CUCRL} for exploration is choose both routes uniformly at random. Figure \ref{fig:grid_world_regret_cost} show the cumulative regret and average cost of the {\CUCRL} and {\RSUCRL} algorithm and Figure \ref{subfig:grid_world_policy} shows the policy learned by them. As we can see, {\RSUCRL} either learns to only choose orange or blue route respectively, causing either constraint violation or large reward regret, while {\CUCRL} converges to the optimal policy.


\begin{figure}[t]
	\centering
	\subfigure[]{\label{subfig:grid_world_2}\includegraphics[width=0.3\linewidth]{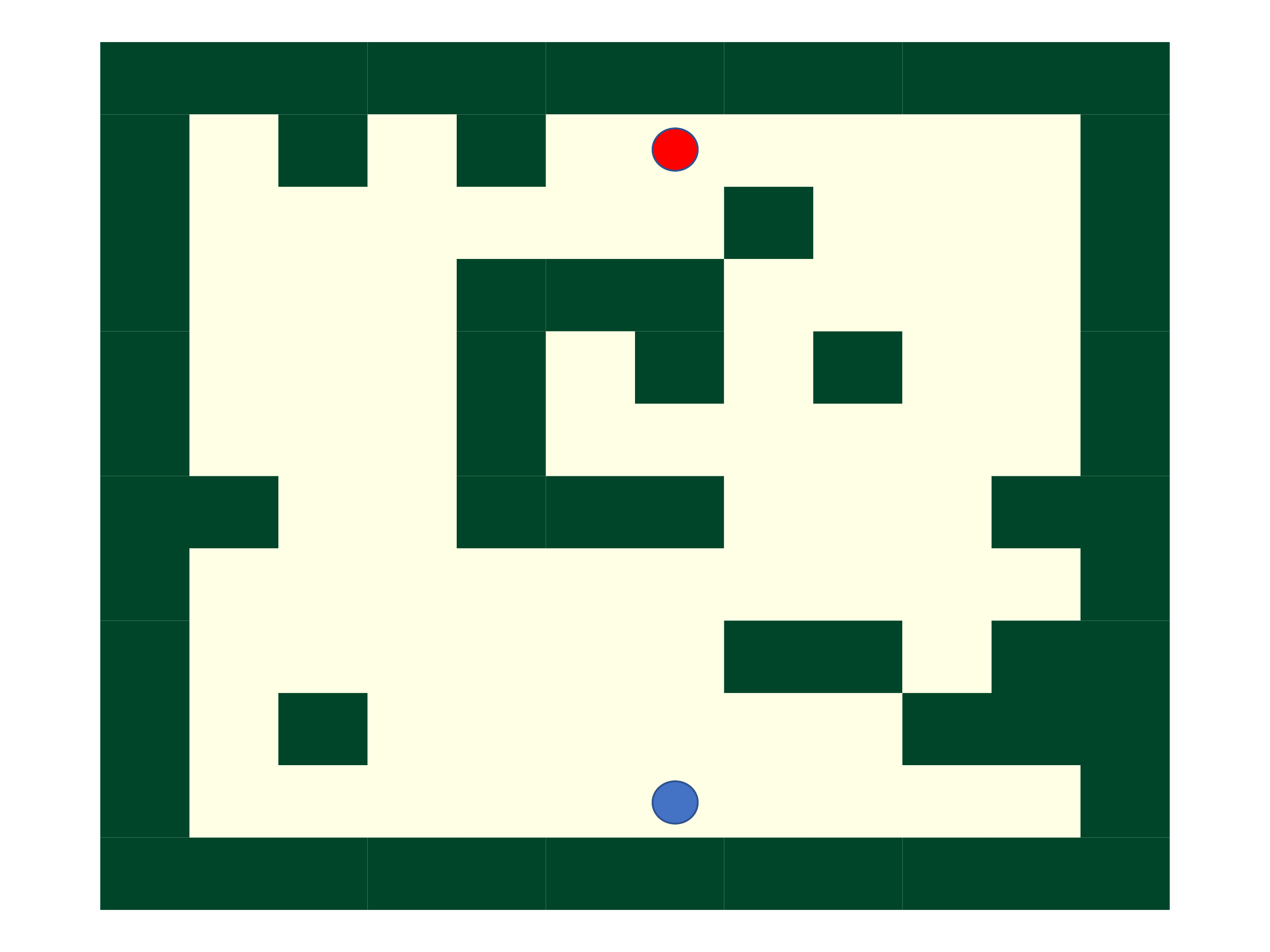}} 
	\subfigure[]{\label{subfig:grid_world_regret_cost_2}\includegraphics[width=0.65\linewidth]{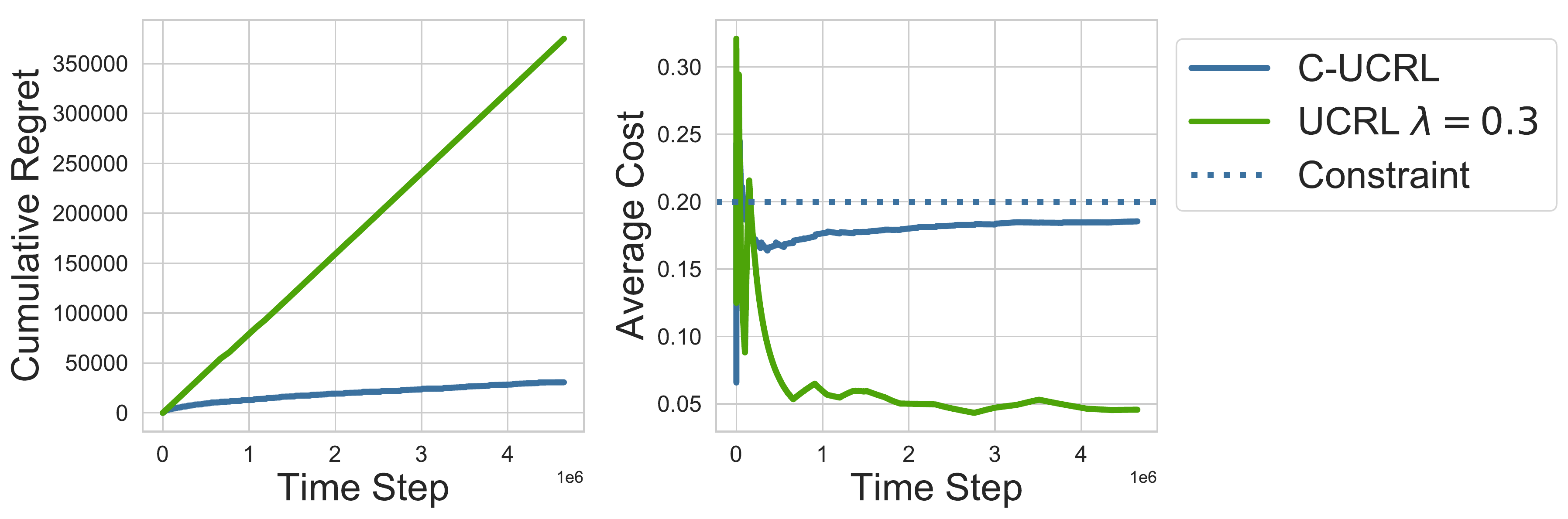}} 
	\caption{Grid World with Safety Constraints. (a) CMDP grid world structure: the states with green color have mean cost equals to $1$ and others have no cost; the blue state  is the origin state and the red state  is the destination state. (b)  Cumulative regret and average reward of {\CUCRL} and {\RSUCRL}. }
	\label{fig:grid_world_2}
\end{figure}

Figure \ref{subfig:grid_world_2} shows the structure of another larger scale safety grid world example. The green states in the figure have mean cost $1$ and the others have zero cost. The blue state is the origin state and the red state is the destination state, which has reward $1$. Figure \ref{subfig:grid_world_regret_cost_2} shows the cumulative regret and average cost of the {\CUCRL} algorithm and {\RSUCRL} algorithm. The {\RSUCRL} algorithm is able to learn a policy that does not violate the constraint if we choose a conservative $\lambda$, however, with much larger reward regret as compared to {\CUCRL}.

\section{Conclusion}
\label{sec:conclusion}

We formulate the problem of safe reinforcement learning when the transition kernel is known but the reward and constraint costs are unknown a priori as a CMDP and propose a {\CUCRL} algorithm to learn the optimal policy. 
Theoretically, we show that {\CUCRL} algorithm is guaranteed to satisfy the constraints during learning with probability at least $1-\delta$ and achieves $ O(T^{\frac{3}{4}}\sqrt{\log(T/\delta)})$ reward regret. Empirically, we provide examples which demonstrate two key properties relative to comparable algorithms: 1) {\CUCRL} is able to learn the optimal policy which in general is a randomized policy as opposed to a deterministic policy, and 2) {\CUCRL} has high-probability guarantees on remaining safe while learning.

Let us comment briefly on some of the limitations of our approach and avenues for future research. First, we remark that artful selection of the baseline policy $\pi_0$ and the duration $h$ for executing it in each episode is required. We choose $h$ based on the mixing time of the Markov chain induced by $\pi_0$. The choice of these two facets is really central to the algorithm as it defines the exploration phase and hence, the robust linear program that we solve for finding $\tilde{\pi}_k$. The baseline and duration need to be chosen such that in each episode the linear program has a non-trivial feasible set. Our results are predicated on this being case; as noted in Section~\ref{sec:algo}, in practice, however, it may not be. To handle this, we suggest the heuristic of executing the baseline policy in episode $k=1$ until $\tilde{c}_1^\top y_0\leq d$. 
A better understanding of how to ensure that in each episode the feasible set remains non-trivial is an avenue of future work.

We note also that it is likely that {\CUCRL} has a much worse sample complexity as compared to approaches which do not impose any criteria on safe learning during the exploration period. Better understanding of this trafeoff is an avenue for future work. Furthermore, our approach requires knowledge of the transition kernel. It is not immediately obvious how to extend classical approaches such as {\UCRL}, without further exacerbating sample complexity issues, due to the fact that central proof technique we employ is the robust linear programming formulation in each episode which is used to obtain a policy based on the confidence bounds. Alternative approaches may be better suited if the transition kernel is unknown. Another interesting direction that arose in our study of CMDPs is that there is potential to extend the theoretical results of {\UCRL} to {\RSUCRL} through a primal-dual lens for capturing the hyper-parameter $\lambda$; investigating this direction may lead to an alternative for addressing the unknown transition kernel setting, however, the issue of the optimal policy being non-deterministic for the true underlying CMDP and the fact that {\UCRL} seeks out deterministic policies remains.


\bibliography{reference}

\begin{thebibliography}{28}
\providecommand{\natexlab}[1]{#1}
\providecommand{\url}[1]{\texttt{#1}}
\expandafter\ifx\csname urlstyle\endcsname\relax
  \providecommand{\doi}[1]{doi: #1}\else
  \providecommand{\doi}{doi: \begingroup \urlstyle{rm}\Url}\fi

\bibitem[Achiam et~al.(2017)Achiam, Held, Tamar, and
  Abbeel]{achiam2017constrained}
Joshua Achiam, David Held, Aviv Tamar, and Pieter Abbeel.
\newblock Constrained policy optimization.
\newblock \emph{arXiv preprint arXiv:1705.10528}, 2017.

\bibitem[Altman(1999)]{altman1999constrained}
Eitan Altman.
\newblock \emph{Constrained Markov decision processes}, volume~7.
\newblock CRC Press, 1999.

\bibitem[Amodei et~al.(2016)Amodei, Olah, Steinhardt, Christiano, Schulman, and
  Man{\'e}]{amodei2016concrete}
Dario Amodei, Chris Olah, Jacob Steinhardt, Paul Christiano, John Schulman, and
  Dan Man{\'e}.
\newblock Concrete problems in ai safety.
\newblock \emph{arXiv preprint arXiv:1606.06565}, 2016.

\bibitem[Berkenkamp et~al.(2017)Berkenkamp, Turchetta, Schoellig, and
  Krause]{berkenkamp2017safe}
Felix Berkenkamp, Matteo Turchetta, Angela Schoellig, and Andreas Krause.
\newblock Safe model-based reinforcement learning with stability guarantees.
\newblock In \emph{Advances in Neural Information Processing Systems}, pages
  908--918, 2017.

\bibitem[Bhatnagar and Lakshmanan(2012)]{bhatnagar2012online}
Shalabh Bhatnagar and K~Lakshmanan.
\newblock An online actor--critic algorithm with function approximation for
  constrained markov decision processes.
\newblock \emph{Journal of Optimization Theory and Applications}, 153\penalty0
  (3):\penalty0 688--708, 2012.

\bibitem[Bhatnagar et~al.(2009)Bhatnagar, Sutton, Ghavamzadeh, and
  Lee]{bhatnagar2009natural}
Shalabh Bhatnagar, Richard~S Sutton, Mohammad Ghavamzadeh, and Mark Lee.
\newblock Natural actor--critic algorithms.
\newblock \emph{Automatica}, 45\penalty0 (11):\penalty0 2471--2482, 2009.

\bibitem[Cheng et~al.(2019)Cheng, Orosz, Murray, and Burdick]{cheng2019end}
Richard Cheng, G{\'a}bor Orosz, Richard~M Murray, and Joel~W Burdick.
\newblock End-to-end safe reinforcement learning through barrier functions for
  safety-critical continuous control tasks.
\newblock \emph{arXiv preprint arXiv:1903.08792}, 2019.

\bibitem[Chow et~al.(2018{\natexlab{a}})Chow, Ghavamzadeh, Janson, and
  Pavone]{chow2018risk}
Yinlam Chow, Mohammad Ghavamzadeh, Lucas Janson, and Marco Pavone.
\newblock Risk-constrained reinforcement learning with percentile risk
  criteria.
\newblock \emph{Journal of Machine Learning Research}, 18\penalty0
  (167):\penalty0 1--51, 2018{\natexlab{a}}.

\bibitem[Chow et~al.(2018{\natexlab{b}})Chow, Nachum, Duenez-Guzman, and
  Ghavamzadeh]{chow2018lyapunov}
Yinlam Chow, Ofir Nachum, Edgar Duenez-Guzman, and Mohammad Ghavamzadeh.
\newblock A lyapunov-based approach to safe reinforcement learning.
\newblock \emph{arXiv preprint arXiv:1805.07708}, 2018{\natexlab{b}}.

\bibitem[Chung et~al.(2012)Chung, Lam, Liu, and
  Mitzenmacher]{chung2012chernoff}
Kai-Min Chung, Henry Lam, Zhenming Liu, and Michael Mitzenmacher.
\newblock Chernoff-hoeffding bounds for markov chains: Generalized and
  simplified.
\newblock \emph{arXiv preprint arXiv:1201.0559}, 2012.

\bibitem[Coraluppi and Marcus(1999)]{coraluppi1999risk}
Stefano~P Coraluppi and Steven~I Marcus.
\newblock Risk-sensitive and minimax control of discrete-time, finite-state
  markov decision processes.
\newblock \emph{Automatica}, 35\penalty0 (2):\penalty0 301--309, 1999.

\bibitem[Ding et~al.(2013)Ding, Qin, Zhang, and Liu]{ding2013multi}
Wenkui Ding, Tao Qin, Xu-Dong Zhang, and Tie-Yan Liu.
\newblock Multi-armed bandit with budget constraint and variable costs.
\newblock In \emph{Twenty-Seventh AAAI Conference on Artificial Intelligence},
  2013.

\bibitem[El~Chamie et~al.(2019)El~Chamie, Yu, A{\c{c}}{\i}kme{\c{s}}e, and
  Ono]{el2019controlled}
Mahmoud El~Chamie, Yue Yu, Beh{\c{c}}et A{\c{c}}{\i}kme{\c{s}}e, and Masahiro
  Ono.
\newblock Controlled markov processes with safety state constraints.
\newblock \emph{IEEE Transactions on Automatic Control}, 64\penalty0
  (3):\penalty0 1003--1018, 2019.

\bibitem[Garc{\i}a and Fern{\'a}ndez(2015)]{garcia2015comprehensive}
Javier Garc{\i}a and Fernando Fern{\'a}ndez.
\newblock A comprehensive survey on safe reinforcement learning.
\newblock \emph{Journal of Machine Learning Research}, 16\penalty0
  (1):\penalty0 1437--1480, 2015.

\bibitem[Jaksch et~al.(2010)Jaksch, Ortner, and Auer]{jaksch2010near}
Thomas Jaksch, Ronald Ortner, and Peter Auer.
\newblock Near-optimal regret bounds for reinforcement learning.
\newblock \emph{Journal of Machine Learning Research}, 11\penalty0
  (Apr):\penalty0 1563--1600, 2010.

\bibitem[Joseph et~al.(2016)Joseph, Kearns, Morgenstern, and
  Roth]{joseph2016fairness}
Matthew Joseph, Michael Kearns, Jamie~H Morgenstern, and Aaron Roth.
\newblock Fairness in learning: Classic and contextual bandits.
\newblock In \emph{Advances in Neural Information Processing Systems}, pages
  325--333, 2016.

\bibitem[Koller et~al.(2018)Koller, Berkenkamp, Turchetta, and
  Krause]{koller2018learning}
Torsten Koller, Felix Berkenkamp, Matteo Turchetta, and Andreas Krause.
\newblock Learning-based model predictive control for safe exploration.
\newblock In \emph{2018 IEEE Conference on Decision and Control (CDC)}, pages
  6059--6066. IEEE, 2018.

\bibitem[Leike et~al.(2017)Leike, Martic, Krakovna, Ortega, Everitt, Lefrancq,
  Orseau, and Legg]{leike2017ai}
Jan Leike, Miljan Martic, Victoria Krakovna, Pedro~A Ortega, Tom Everitt,
  Andrew Lefrancq, Laurent Orseau, and Shane Legg.
\newblock Ai safety gridworlds.
\newblock \emph{arXiv preprint arXiv:1711.09883}, 2017.

\bibitem[Levine et~al.(2016)Levine, Finn, Darrell, and Abbeel]{levine2016end}
Sergey Levine, Chelsea Finn, Trevor Darrell, and Pieter Abbeel.
\newblock End-to-end training of deep visuomotor policies.
\newblock \emph{The Journal of Machine Learning Research}, 17\penalty0
  (1):\penalty0 1334--1373, 2016.

\bibitem[Lillicrap et~al.(2015)Lillicrap, Hunt, Pritzel, Heess, Erez, Tassa,
  Silver, and Wierstra]{lillicrap2015continuous}
Timothy~P Lillicrap, Jonathan~J Hunt, Alexander Pritzel, Nicolas Heess, Tom
  Erez, Yuval Tassa, David Silver, and Daan Wierstra.
\newblock Continuous control with deep reinforcement learning.
\newblock \emph{arXiv preprint arXiv:1509.02971}, 2015.

\bibitem[Luenberger et~al.(1984)Luenberger, Ye, et~al.]{luenberger1984linear}
David~G Luenberger, Yinyu Ye, et~al.
\newblock \emph{Linear and nonlinear programming}, volume~2.
\newblock Springer, 1984.

\bibitem[Mnih et~al.(2015)Mnih, Kavukcuoglu, Silver, Rusu, Veness, Bellemare,
  Graves, Riedmiller, Fidjeland, Ostrovski, et~al.]{mnih2015human}
Volodymyr Mnih, Koray Kavukcuoglu, David Silver, Andrei~A Rusu, Joel Veness,
  Marc~G Bellemare, Alex Graves, Martin Riedmiller, Andreas~K Fidjeland, Georg
  Ostrovski, et~al.
\newblock Human-level control through deep reinforcement learning.
\newblock \emph{Nature}, 518\penalty0 (7540):\penalty0 529, 2015.

\bibitem[Moldovan and Abbeel(2012)]{moldovan2012safe}
Teodor~Mihai Moldovan and Pieter Abbeel.
\newblock Safe exploration in markov decision processes.
\newblock \emph{arXiv preprint arXiv:1205.4810}, 2012.

\bibitem[Puterman(2014)]{puterman2014markov}
Martin~L Puterman.
\newblock \emph{Markov decision processes: discrete stochastic dynamic
  programming}.
\newblock John Wiley \& Sons, 2014.

\bibitem[Sallab et~al.(2017)Sallab, Abdou, Perot, and Yogamani]{sallab2017deep}
Ahmad~EL Sallab, Mohammed Abdou, Etienne Perot, and Senthil Yogamani.
\newblock Deep reinforcement learning framework for autonomous driving.
\newblock \emph{Electronic Imaging}, 2017\penalty0 (19):\penalty0 70--76, 2017.

\bibitem[Shani et~al.(2005)Shani, Heckerman, and Brafman]{shani2005mdp}
Guy Shani, David Heckerman, and Ronen~I Brafman.
\newblock An mdp-based recommender system.
\newblock \emph{Journal of Machine Learning Research}, 6\penalty0
  (Sep):\penalty0 1265--1295, 2005.

\bibitem[Wachi et~al.(2018)Wachi, Sui, Yue, and Ono]{wachi2018safe}
Akifumi Wachi, Yanan Sui, Yisong Yue, and Masahiro Ono.
\newblock Safe exploration and optimization of constrained mdps using gaussian
  processes.
\newblock In \emph{Thirty-Second AAAI Conference on Artificial Intelligence},
  2018.

\bibitem[Zhou and Tomlin(2018)]{zhou2018budget}
Datong~P Zhou and Claire~J Tomlin.
\newblock Budget-constrained multi-armed bandits with multiple plays.
\newblock In \emph{Thirty-Second AAAI Conference on Artificial Intelligence},
  2018.

\end{thebibliography}

\end{document}